\documentclass[preprint,12pt]{elsarticle}

\usepackage{amssymb}
\usepackage{amsmath}
\usepackage{amsthm}
\usepackage{mathtools} 
\usepackage{algorithmicx}
\usepackage{algorithm}
\usepackage[noend]{algpseudocode}
\usepackage{hyperref}
\usepackage{placeins}

\makeatletter
\def\ps@pprintTitle{%
\let\@oddhead\@empty
\let\@evenhead\@empty
\def\@oddfoot{}%
\let\@evenfoot\@oddfoot}
\makeatother

\newtheorem{definition}{Definition}
\newtheorem{lemma}{Lemma}
\newtheorem{theorem}{Theorem}
\newtheorem{require}{Requirement}
\newtheorem{arcdef}{ARC Definition}

\newcommand\pfun{\mathrel{\ooalign{\hfil$\mapstochar$\hfil\cr$\to$\cr}}}

\DeclareMathOperator*{\argmin}{arg\,min}

\newenvironment{brsm}{
  \bigl[ \begin{smallmatrix} }{%
  \end{smallmatrix} \bigr]}

\usepackage{todonotes}

\newcommand{\HIDE}[1]{}

\begin{document}

\begin{frontmatter}


  
\title{MADIL: An MDL-based Framework for Efficient Program Synthesis in the ARC Benchmark}


\author{Sébastien Ferré} 

\affiliation{organization={Univ Rennes, CNRS, Inria, IRISA},
            addressline={\\263 av. Général Leclerc}, 
            city={Rennes},
            postcode={35042}, 
            country={France}}
          \ead{ferre@irisa.fr}
          
\begin{abstract}
  Artificial Intelligence (AI) has achieved remarkable success in
  specialized tasks but struggles with efficient skill acquisition and
  generalization. The Abstraction and Reasoning Corpus (ARC) benchmark
  evaluates intelligence based on minimal training requirements. While
  Large Language Models (LLMs) have recently improved ARC performance,
  they rely on extensive pretraining and high computational costs. We
  introduce MADIL (MDL-based AI), a novel approach leveraging the
  Minimum Description Length (MDL) principle for efficient inductive
  learning. MADIL performs pattern-based decomposition, enabling
  structured generalization. While its performance (7\% at ArcPrize
  2024) remains below LLM-based methods, it offers greater efficiency
  and interpretability. This paper details MADIL’s methodology, its
  application to ARC, and experimental evaluations.
\end{abstract}



\begin{keyword}
  Artificial Intelligence \sep Inductive Learning \sep Program Synthesis \sep Minimum Description Length \sep Abstraction and Reasoning Corpus \sep Pattern-based Decomposition




\end{keyword}

\end{frontmatter}



\section{Introduction}
\label{intro}

Over the past decade, Artificial Intelligence (AI) has achieved
remarkable success in specialized tasks, often surpassing human
performance in domains such as image
recognition~\cite{KriSutHin2012nips} and board
games~\cite{Silver2016alphago}. However, despite these advances, AI
remains limited in its ability to generalize and adapt to novel tasks
with minimal training -- a hallmark of human intelligence. To
encourage progress beyond narrow task-specific
generalization~\cite{Goertzel2014}, F. Chollet proposed a new measure
of intelligence that prioritizes {\em skill-acquisition efficiency}
over {\em skill performance}~\cite{Chollet2020}. In this framework,
intelligence is defined by the amount of prior knowledge and
experience an agent requires to achieve competence across a diverse
range of tasks, rather than its peak performance in any single domain.

To empirically assess this notion of intelligence, Chollet introduced
the Abstraction and Reasoning Corpus (ARC, aka. ARC-AGI), a benchmark
designed as a psychometric test for evaluating and comparing human and
machine intelligence. ARC consists of a collection of tasks that
require learning transformation rules for colored grids based on very
limited input-output examples (3.3 on
average). Figures~\ref{fig:example:other} and \ref{fig:example}
illustrate two sample ARC tasks, with the second serving as a running
example throughout this paper.

\begin{figure}[t]
  \centering
  \includegraphics[width=0.8\textwidth]{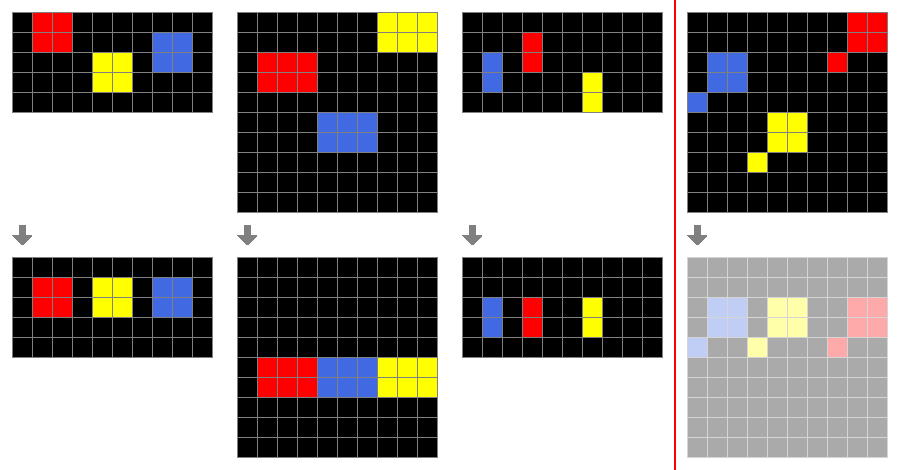}
  \caption{Task 1caeab9d (inputs at the top, outputs at the bottom, test on the right)}
  \label{fig:example:other}
\end{figure}

The Abstraction and Reasoning Corpus (ARC) presents a significant
challenge for AI systems. While humans can solve over 80\% of the
tasks~\cite{Johnson2021}, AI performance has lagged far behind. The
winner of the Kaggle 2020
competition\footnote{\url{https://www.kaggle.com/c/abstraction-and-reasoning-challenge}}
managed to solve only 20\% of the tasks, relying heavily on hard-coded
primitives and brute-force search.
%
A major breakthrough came with the use of Large Language Models (LLMs)
for predicting missing grids, an approach pioneered by J. Cole at
MindsAI. This innovation led to rapid progress, reaching 30\% accuracy
at ARCathon
2023\footnote{\url{https://lab42.global/past-challenges/arcathon-2022/}}
and 55\% at ArcPrize 2024\footnote{
  \url{https://arcprize.org/}}. By the end of
2024, OpenAI's o3 reasoning LLM achieved human-level performance,
albeit at an extremely high computational cost -- running the
high-compute version required thousands of dollars per task at
inference time.

Despite these advances, LLM-based methods rely on extensive
pre-training with millions of synthesized tasks and require
substantial computational resources for fine-tuning and reasoning
during inference. As a result, the challenge of efficient skill
acquisition remains largely unresolved. In March 2025, a new version
of the benchmark,
ARC-AGI-2\footnote{\url{https://arcprize.org/blog/announcing-arc-agi-2-and-arc-prize-2025}},
was introduced to push AI research further. While still simple to most
human solvers, ARC-AGI-2 presents significantly greater difficulty for
LLMs -- demonstrated by o3’s success rate dropping to just a few
percent.

In this paper, we present MADIL, an alternative approach to the ARC
benchmark. MADIL is a general framework for inductive learning from
small sets of input-output examples. While we focus on its application
to ARC, the method has also been successfully applied to
string-to-string transformation tasks, such as those in
FlashFill~\cite{FlashFill2013}.
MADIL, which stands for ``MDL-based AI,'' is founded on the Minimum
Description Length (MDL) principle -- a concept from information
theory that states: ``The model that best describes the data is the
one that compresses it the most''~\cite{Rissanen1978,Grunwald2019}. In
essence, MADIL searches for task models that both explain the given
examples concisely and generalize well to unseen inputs.
At inference time, MADIL operates by pattern-based decomposition,
breaking down an input into meaningful subcomponents in a top-down
manner, then constructing the corresponding output bottom-up. During
learning, it identifies optimal decompositions for both inputs and
outputs and determines the transformations between corresponding
parts. The key advantage of this approach is that the MDL principle
guides the decomposition process -- favoring representations that
achieve greater compression -- while also simplifying the residual
part-to-part transformations, making them easier to learn.

Although MADIL's overall performance remains below state-of-the-art
methods -- improving from 2\% at ARCathon 2022 to 7\% at ArcPrize 2024
-- its MDL-based search is highly efficient, enabling the discovery of
complex models in under a minute (on a single CPU). Unlike brute-force
search approaches that perform a wide but shallow exploration, MADIL
conducts a narrow but deep search, with most solutions found early
along the first exploration path.
Compared to LLM-based approaches, MADIL does not require synthetic
task generation or data augmentation. Instead, it leverages Core
Knowledge priors, encoded as a set of primitives, patterns and
functions. While these primitives are domain-specific -- designed for
reasoning over colored grids -- many are broadly applicable beyond ARC
tasks. Examples include arithmetic and bitwise operations, geometric
transformations, and collection manipulations, making MADIL a more
structured and interpretable alternative to data-intensive deep
learning methods.

This paper significantly expands on our previous work on
MADIL~\cite{Fer2023dexa,Fer2024ida}, providing deeper explanations,
presenting a more advanced solution to ARC, and introducing several
novel contributions.
Section~\ref{arc} offers a quick overview of the ARC
benchmark. Section~\ref{related} discusses related work, covering
existing approaches to ARC and broader research in program
synthesis. Section~\ref{sec:overview} introduces our approach through
a concrete example task. Section~\ref{madil} formalizes the general
theoretical framework of MADIL and demonstrates its application to
ARC. Section~\ref{algos} describes the key algorithms and other
practical aspects.
Section~\ref{sec:advanced} presents three recent enhancements to
MADIL: (1) the integration of collection management, (2) the use of
dependent patterns, and (3) the application of Monte Carlo Tree Search
as an improvement over greedy search. Section~\ref{eval} reports
experimental results evaluating MADIL’s performance, efficiency, and
limitations. Finally, Section~\ref{conclu} summarizes our findings and
outlines directions for future research.


\section{Abstraction and Reasoning Corpus (ARC)} 
\label{arc}

\begin{figure}[t]
  \centering
  \includegraphics[width=0.8\textwidth]{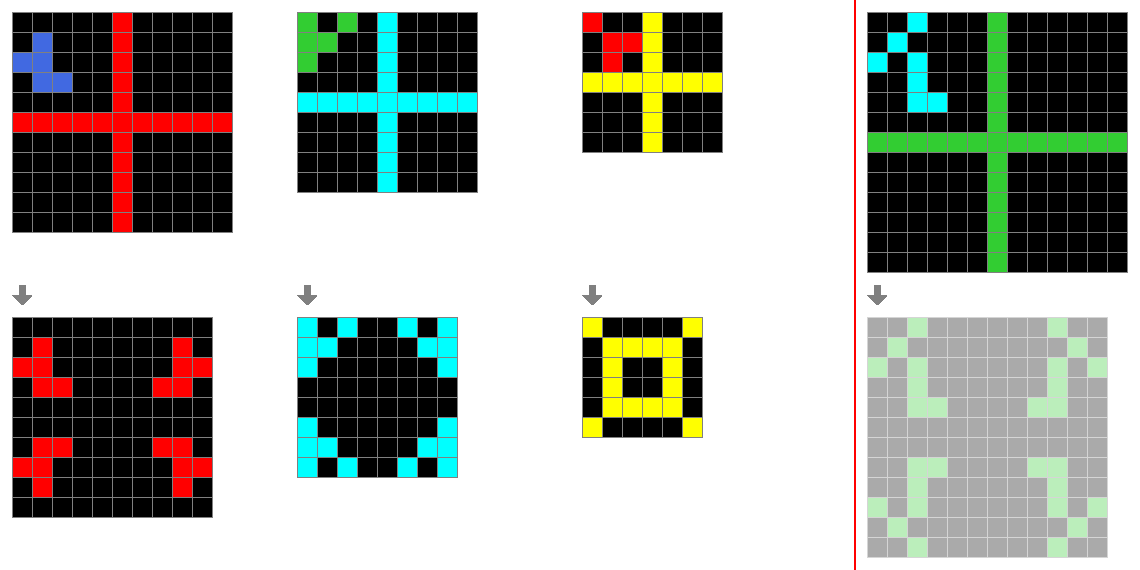}
  \caption{Task 47c1f68c (inputs at the top, outputs at the bottom, test on the right)}
  \label{fig:example}
\end{figure}

ARC is a collection of tasks\footnote{Data and testing interface at
  \url{https://github.com/fchollet/ARC}}, where each task is made of
training examples (3.3 on average) and test examples (1 in
general). Each example is made of an input grid and an output
grid. Each grid is a 2D array (with size up to 30x30) filled with
integers coding for colors (there are 10 distinct colors). For a given
task, the size of grids can vary from one example to another, and
between the input and the output.
Each task is a machine learning problem, whose goal is to learn a
model that can generate the output grid from the input grid, and so
from a few training examples only. Prediction is successful only if
the predicted output grid is {\em strictly equal} to the expected grid
for {\em all} test examples, there is no partial success. However,
three attempts\footnote{The number of attempts was lowered to 2 in
  ArcPrize.} are allowed for each test example to compensate for
potential ambiguities in the training examples.
Figure~\ref{fig:example} shows an ARC task that is used as a running
example in the following. The grid to be predicted is the one at the
bottom right.

ARC is composed of 1000 tasks in total: 400 ``training
tasks''\footnote{The term ``training tasks'' may be misleading as
  their purpose is to train AI developers, not AI systems. Humans
  solve ARC tasks without training. \HIDE{Indeed, ARC tasks can be
    solved by humans without any ARC-specific training.}},
400 evaluation tasks (aka. public tasks), and 200 secret tasks for
independent evaluation. Among secret tasks, 100 form the so-called
private set that was used in the challenges from 2020 to 2024 (Kaggle,
ARCathon, and ArcPrize), and the rest form the so-called semi-private
set that was used to evaluate proprietary LLMs (see ArcPrize).
Developers should only look at the training tasks, not at the
evaluation tasks. The latter should only be used to evaluate the broad
generalization capability of the developed systems.

\section{Related Work}
\label{related}

Earlier approaches to ARC define a DSL (Domain-Specific Language) of
programs -- based on function composition -- that transform an input
grid into an output grid, and search for a program that is correct on
the training examples. The differences mostly lie in the primitive
functions (prior knowledge) and in the search strategy. \HIDE{It is
  tempting to define more and more functions like the Kaggle winner
  did, hence more prior knowledge, but this means a less intelligent
  system according to Chollet's measure.}
Brute-force search led to some success -- it is the approach taken by
the winners at Kaggle'20 (Icecuber) and ARCathon'22 (Michael Hodel)
competitions -- but this cannot be a solution to ARC and AGI. To guide
the search in the huge program space, other approaches use grammatical
evolution~\cite{Fischer2020}, neural networks~\cite{Alford2021cnta,Ouellette2024},
search tree pruning with hashing and Tabu list~\cite{Xu2022}, or
stochastic search trained on solved tasks~\cite{Ainooson2023}.
A difficulty is that the output grids are generally only used to score
a candidate program so that the search is kind of
blind. Ouellette~\cite{Ouellette2024} and Alford~\cite{Alford2021cnta}
improve this with a neural-guided search to take the ouput grid into
account in the choice of the search steps,
and Xu~\cite{Xu2022} compares the in-progress generated grid to the
expected grid. However, this assumes that output grids are comparable
to input grids, which is not true for all tasks. Function-based DSL
approaches have a scaling issue because the search space increases
exponentially with the number of primitive functions. For this reason,
search depth is often bounded by~3 or~4. Ainooson~\cite{Ainooson2023}
alleviates this difficulty by defining high-level functions that
embody specialized search strategies.
Most approaches based on DSL design and search scored under 10\% on
the public and private sets, with the notable exception of Icecuber's
approach that scored 20.6\% at Kaggle'20. A key ingredient of its
success seems to be the decomposition of the output grids by stacking
layers taken from a large collection of pieces computed from the input
grids.

Later approaches use Large Language Models (LLM) to generate output
grids or transformation programs, achieving a major progress by
scoring up to 30\% at ARCathon'23, and 55\% at ArcPrize'24. Actually,
general-purpose LLMs such as GPT-4o perform poorly on ARC tasks. The
approach pioneered by MindsAI consists in synthesizing a very large
set of ARC-like tasks, and to train a specialized LLM on them. Another
essential ingredient is Test-Time Fine-Tuning (TTFT). It consists in
augmenting the few examples of an ARC task into thousands of examples,
and then fine-tuning the LLM to the task before generating many
candidate output grids, and voting for the most promising ones.
Greenblatt~\cite{Greenblatt2024} and Berman~\cite{Berman2024} adopted
an inductive rather than transductive approach. Instead of asking the
LLM to directly generate output grids for a test input grid, they ask
the LLM to reason on the task examples in order to generate thousands
of candidate transformation programs (e.g., as Python code). They evaluate
those programs by evaluating them on the examples, and they adopt an
evolutionary approach where they ask the LLM to revise the more
promising programs into successive generations.
Despite the objective success of LLM-based approaches on ARC, there
are questions about the actual progress in terms of AGI. First, the
LLMs have been heavily trained on millions of ARC-like tasks. There is
a risk that, for some private tasks, there are synthetic tasks that are
very similar so that the LLM would only need to ``retrieve'' the
solution rather than ``reason'' on a new task. For recall, ARC was
designed to test out-of-distribution inference. Moreover, humans can
solve ARC tasks without prior exposure to them, solving them from core
knowledge priors only. Second, test-time compute is huge because of
example augmentation, fine-tuning and the generation of thousands of
candidates. Efficiency was identified as an important factor of
intelligence~\cite{Chollet2019}, the opposite of brute-force search,
and massive LLM-based generation can be assimilated to a form of
brute-force search.

Johnson {\em et al.}~\cite{Johnson2021} report on a psychological
study of ARC. It reveals that humans use object-centric mental
representations to solve ARC tasks. This is in contrast with existing
solutions that are based on grid transformations. Interestingly, the
tasks that are found the most difficult by humans are those based on
logics (e.g., an exclusive-or between grids) and symmetries (e.g.,
rotation), precisely those most easily solved by transformation-based
approaches.
The study exhibits two challenges: (1) the need for a large set of
primitives, especially about geometry; (2) the difficulty to identify
objects, which can be only visible in part due to overlap or
occlusion.
A valuable resource is LARC, for Language-annotated
ARC~\cite{Acquaviva2022}, collected by crowd-sourcing. It provides for
most training tasks one or several {\em natural programs}.
They are natural in that they are short natural language texts
produced by humans trying to solve ARC tasks. They are programs in
that they were proved to be effective by involving two separated
participants: a describer that produces the text given the training
examples only, and a builder that generates the output grid given the
produced text and the test input grid only.
Those natural programs confirm the object-centric and declarative
nature of human representations.

Beyond the ARC benchmark, a number of work has been done in the domain
of {\em program synthesis}, which is also known as program induction
or programming by examples (PbE)~\cite{Lieberman2001pbe}.  An early
approach is Inductive Logic Programming (ILP)~\cite{MugRae1994}, where
target predicates are learned from symbolic representations.
A more recent success story in program synthesis is
FlashFill~\cite{FlashFill2013}. It generates string-to-string programs
from a few examples, and has been deployed in Microsoft Excel to
automatically fill columns after the user has provided values for a
few rows. It relies on the definition of a DSL and clever
datastructures for computing the set of all programs that are
compatible with the
examples. FlashMeta~\cite{Polozov2015flashmeta}. is a framework that
generalizes FlashFill, facilitating application to other DSLs A key
ingredient of FlashMeta is {\em witness functions} that capture the
inverse semantics of DSL primitives, and hence enable to take into
account example outputs -- to some extent -- in the search for
programs. Scrybe~\cite{Mulleners2023} also features a kind of inverse
semantics where examples are propagated backward from the result of a
function to its arguments. However, it seems limited to structured
data (lists of integers) and combinator programs (mostly list
filtering and permutation). Rule et al~\cite{Rule2024} also consider
list-to-list programs but searches a space of {\em metaprograms}
rather than the space of programs directly. A metaprogram transforms a
set of examples into a program, in a few steps. It uses orders of
magnitude less search, and reaches performance close to humans.
Dreamcoder~\cite{Ellis2021dreamcoder} alternates a {\em wake} phase
that uses a neurally guided search to solve tasks, and a {\em sleep}
phase that extends a library of abstractions to compress programs
found during the wake phase. In some tasks, Bayesian program learning
was shown to outperform deep learning, e.g. for parsing and generating
handwritten world's alphabets~\cite{LakSalTen2015}.

\section{Overview of the Proposed Approach}
\label{sec:overview}


In this section we give an informal presentation of our approach
before diving into the technical details in the next sections.  We
base this presentation on a task taken from the training set, task
47c1f68c. Figure~\ref{fig:example} shows the three demonstration
examples on the left, and the test example on the right; input grids
are at the top while output grids are at the bottom. We have chosen a
relatively complex task that cannot be easily solved by brute-force
search in a DSL search space, and in which both outputs and inputs
exhibit some structure.

\begin{figure}[tt]
  \centering
  \includegraphics[width=\textwidth]{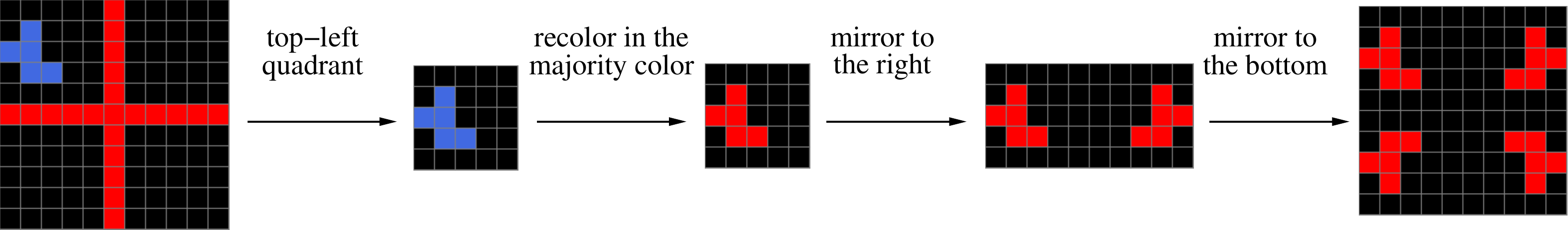}
  \caption{Solution with a sequence of transformations.}
  \label{fig:example:sequence}
\end{figure}

Existing approaches based on program synthesis typically define
programs as sequences of transformations, from inputs to
outputs. Figure~\ref{fig:example:sequence} shows such a sequence of
transformations solving the running task, and its application to the
first demonstration example. The sequence is made of 4 steps: crop the
top-left quadrant, then recolor the non-black cells in the majority
color of the input grid, and finally mirror twice the result, to the
right and to the bottom.
Several observations can be made about the search for a solution
sequence of transformations:
\begin{itemize}
\item Among all the transformations that are applicable to a grid
  (input or intermediate), there is no reason to prefer one or another
  independently of the target output. A good transformation is a
  transformation that is useful to the generation of the output.
\item In general, there is no clear way to assess the usefulness of an
  intermediate grid w.r.t. the target output. A good intermediate grid
  may have a different size, different colors, or contain different
  shapes.
\item Most transformations are not invertible so that in most
  approaches sequences are generated and evaluated from the input to
  the output only.
\end{itemize}
For those reasons, full sequences are typically generated before being
evaluated by comparing the predicted outputs with the expected
ones. This makes enumerative search exponential and therefore limited
to short sequences. More advanced approaches use a trained model to
predict promising transformations for the next step given the input,
output, and previous steps.

\begin{figure}[t]
  \centering
  \includegraphics[width=\textwidth]{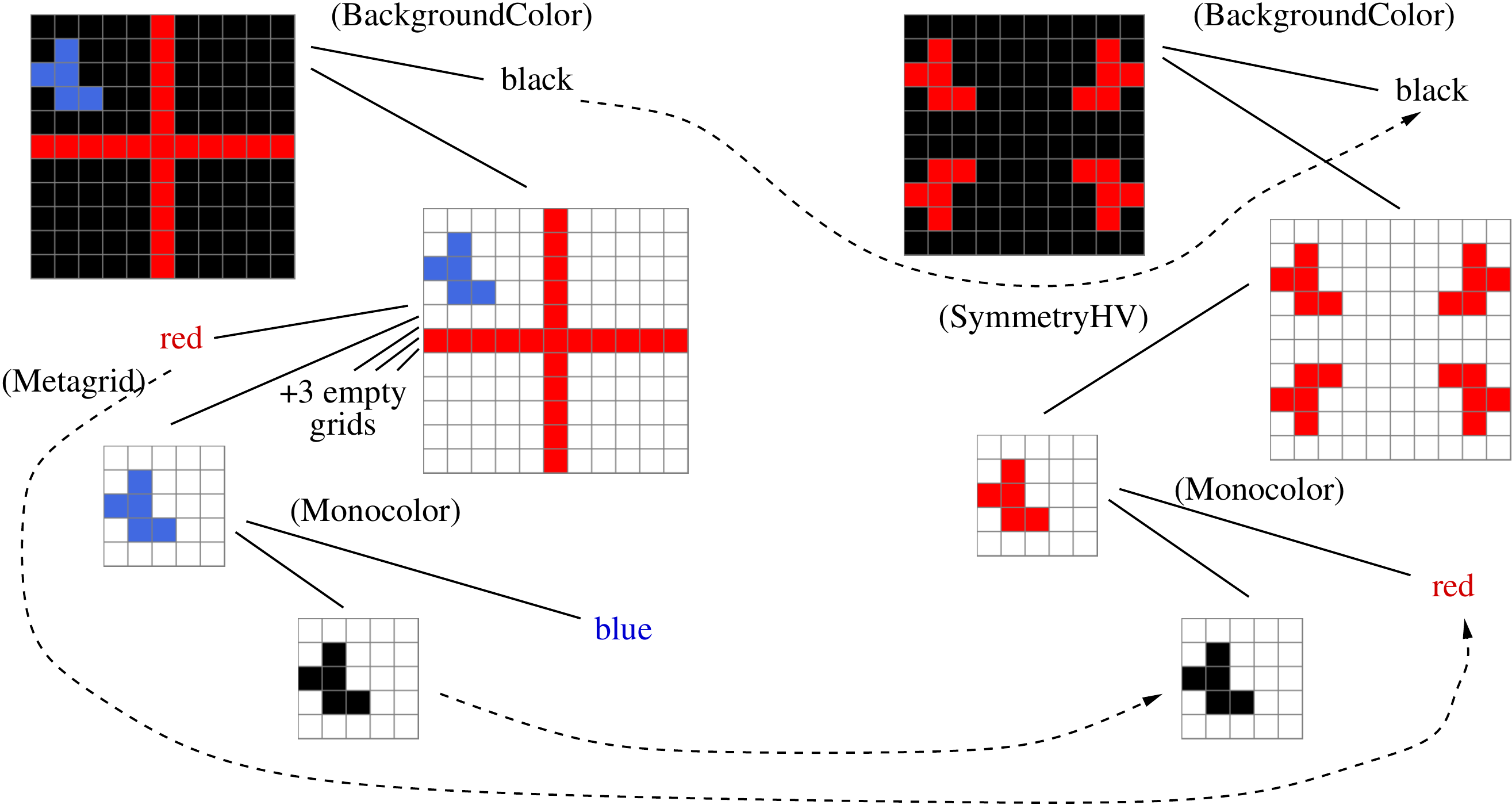}
  \caption{Solution with a decomposition/composition of the grids (solid lines), and mappings from input parts to output parts (dashed arrows).}
  \label{fig:example:decomposition}
\end{figure}

\begin{figure}[t]
{\small
\begin{verbatim}
BgColor(bgcolor: ?,
    Metagrid(sepcolor: ?,
        [ [ Monocolor(color: ?, mask: ?), Empty(size2: ?) ],
          [ Empty(size3: ?),              Empty(size4: ?) ] ]))
=>
BgColor(bgcolor = bgcolor,
    SymmetryHV(
        Monocolor(color = sepcolor, mask = mask)))
\end{verbatim}
}
  \caption{A MADIL program that is a solution to the running task.}
  \label{fig:example:model}
\end{figure}

In the MADIL approach, a program first decomposes the input grid into
different parts, and then compose the output grid from those parts.
Figure~\ref{fig:example:decomposition} shows the decompose-compose
process for the running task, illustrated on the first example. The
input grid is decomposed into a background color (here black), and the
grid contents, which is decomposed as a ``metagrid'' of 4 subgrids and a
separator color (here red). The top-left subgrid is further decomposed
into its single color (here blue), and the contents mask. From there,
the output grid is composed from the same background color as the
input (here black), and a grid contents that results from a symmetric
unfolding of a smaller grid that results from the coloring of the
input mask with the input seperator color (here red). Note that this
program generalizes correctly to other background color than
``black'', although there is no example that requires it.

Each decomposition/composition is called a {\em pattern}, which can be
run in both directions: from whole to parts (decomposition), and from
parts to whole (composition). Composition must be deterministic but
decomposition may be underterministic: there may be different ways to
decompose a grid along a pattern, e.g. when a grid matches different
symmetries.
Figure~\ref{fig:example:model} shows the MADIL program that is a
solution to the running task. We actually talk about a {\em task
  model} as it is made of two grid models, where each grid model is a
composition of patterns. In addition, the output model can refer to
input parts (e.g., {\tt sepcolor}) -- and apply functions -- in order
to specify output parts from input parts.

Several observations can be made about the search for a
decompose-compose program:
\begin{itemize}
\item Among all the patterns that are applicable to input grids, there
  is a preference for those that compress the grids more,
  independently of the output grid.
\item The compression rate can be defined by relying on information
  theory, and overfitting can be avoided by relying on two-parts MDL.
\item Patterns are invertible so that the program can be grown from
  both sides, searching for good decompositions of both the input and
  the output grids.
\end{itemize}
The search for a MADIL program can therefore be split into three
components: a compressive decomposition of the input grid, a
compressive decomposition of the output grid, and mappings from the
input parts to the output parts. The latter can be done in the
classical way, searching for transformation sequences from input parts
to output parts. However, shorter sequences can be expected here
because the decompositions contribute to the whole transformation. In
the running task, it can be observed that empty sequences are enough
as each output part is found among the input parts.

\section{Theory}
\label{madil}

In this section we lay the theoretical foundations of our approach. We
first define all key notions, from values to task models, going
through patterns and functions, which are the building blocks of
models (Section~\ref{definitions}). We then define the description
lengths of models and descriptions, in order to apply the MDL
principle. Finally we define the learning goal and the search space of
task models, and establish useful properties over this search space.
Those definitions are fully generic and can in principle be applied to
any kind of inputs and outputs. The notions that need to be
instantiated for a particular domain are values, patterns, and
functions. We provide their ARC definition as illustrative examples,
and as the basis for our experiments.


\subsection{Definitions: from Values to Task Models}
\label{definitions}


\subsubsection{Values}
\label{value}

Values are the basic ingredients of MADIL models. There are values for
the task inputs and outputs, and also for intermediate representations
in the transformation process from input to output.

\begin{definition}[value]
  A {\em value}~$v \in {\cal V}$ is any piece of information, represented using data types. ${\cal V}$ is called the {\em universe}.
\end{definition}

\begin{arcdef}
  For ARC, values are colored grids; integers for positions, sizes, or
  moves; colors; color maps; and {\em motifs} (symmetries, periodic color
  patterns, regular shapes). There are three subtypes of colored
  grids. {\em Sprites} are partially-colored grids, using an
  additional transparent color, in order to represent non-rectangular
  objects. {\em Masks} are non-colored sprites, like bitmaps, in order
  to represent shapes (color black is used for the cells belonging to
  the shape). There are also two types of composites values: vectors
  made of two integers, to represent positions, sizes, and moves; {\em
    objects} made of a position and a sprite.
\end{arcdef}

In order to measure description lengths for values, we also need {\em
  value distributions} from which values are drawn. They also serve to
constrain a value to belong to some subset of values. An example of
distribution is the uniform distribution over non-black colors.

\begin{definition}[value distribution]
  A {\em value distribution} $V$ is defined by its probability mass
  function (pmf) $f_V \in {\cal V} \rightarrow [0,1]$. Its support is
  written $R_V$, defined as $\{ v \in {\cal V} \mid f_V(v) > 0\}$. \HIDE{We
  note $\{v\}$, or simply $v$ for short, the singleton distribution
  where all probability mass is on~$v$.}  We also note $v \sim V$ to
  express the fact that value~$v$ is drawn from distribution~$V$,
  i.e. $f_V(v) > 0$.
\end{definition}

\begin{arcdef}
  We use uniform distributions over an interval for integers, and over
  lists of motifs for motifs. For colors, we also use a list of
  possible values but we distinguish between background colors where
  black is more likely, and object colors where black is less likely.
  For grids, we distinguish the three subtypes of grids, and we use
  integer distributions for the height and width of the grid, and a
  color distribution for the cell colors.
\end{arcdef}

In this work, we rather use description lengths than
probabilities. The optimal encoding theorem of Shannon provides a
direct relationship between the two.

\begin{definition}[value description length]
  Let $V$ be a value distribution, and $v \sim V$ a value drawn from
  this distribution. The {\em description length} of the value in
  bits, relative to the distribution, is defined as
  \[ L(v \mid V) := - \log_2 f_V(v) \]
\end{definition}

Here we do not bother about the actual encoding, only the coding length.
In many cases, a distribution will be defined through encoding rather
than through a pmf, along the equation $f_V(v) = 2^{- L_V(v)}$. It is
valid if the encoding behind the definition of $L(. \mid V)$ is
lossless, as it ensures that the sum of probability masses is at most
1.
For instance, a common encoding of positive integers is the Elias
encoding, where $L(n) = 2 \log_2 n + 1$.

\subsubsection{Functions and Patterns}
\label{function}
\label{pattern}

Computations in MADIL models are performed by functions and patterns.
We assume finite sets of functions and patterns to be used in models.

\begin{definition}[functions]
  We assume a collection~${\cal F}$, where each
  element~$f \in {\cal F}$ is a k-ary partial function on values:
  $f \in {\cal V}^k \pfun {\cal V}$. We note $v = f(v_1, \ldots, v_k)$
  to say that $v$ is the result of applying the function~$f$ to the
  $k$ argument values. We note $f(v_1,\ldots,v_k) = \bot$ when the
  function is undefined on its arguments.
  We note $D_f \subseteq {\cal V}^k$ the domain of definition of the function, and
  $R_f \subseteq {\cal V}$ its range.
\end{definition}

\begin{table}[t]
  \centering
  \begin{tabular}{|p{0.9\columnwidth}|}
    \hline
    {\bf Arithmetics}:
    addition and substraction of a small constant ($0..3$);
    product and division by a small constant ($2..3$);
    vectorized versions of the previous functions; 
    integer components of a vector;
    transposition of a vector. \\
    \hline
    {\bf Geometry}:
    size and area of a grid;
    number of colors of a grid;
    majority and minority colors of a grid;
    halves and quadrants of a grid;
    extracting the content of a sprite as an object;
    mask of a sprite;
    complement of a mask (logical not);
    compressing a grid by removing repeated rows/columns;
    applying symmetries to grids (combining rotations and reflections);
    completing a grid along a symmetry;
    replacing each cell of some color by the grid itself (self-compose). \\
    \hline
    {\bf Objects}:
    position and sprite of an object;
    extremal and median positions of an object along each axis (e.g., top and bottom, middle);
    border, interior and various neighborhoods of objects. \\
    \hline
  \end{tabular}
  \caption{Functions by domain}
  \label{tab:arc:functions}
\end{table}

\begin{arcdef}
  Table~\ref{tab:arc:functions} lists the available ARC functions by
  domain (arithmetics, geometry, and objects). In practice, all
  functions are used with one variable argument, and other arguments
  are set to constants (e.g., small integers, colors, symmetries).
\end{arcdef}

\begin{definition}[patterns]\label{def:patterns}
  We assume a collection~${\cal P}$, where each
  element~$P \in {\cal P}$ is a k-ary {\em pattern}. A pattern is a
  distribution on tuples of values
  $(v,v_1,\ldots,v_k) \in {\cal V}^{k+1}$, establishing a relationship
  between a value~$v$ and its decomposition into $k$ values
  $(v_1, \ldots, v_k)$ such that $v$ is unique given
  $(v_1, \dots, v_k)$ but not necessarily the reverse:
  \[ |\{ v \mid (v,v_1,\ldots,v_k) \sim P \}| \leq 1 \textrm{, for all } (v_1,\ldots,v_k) \in {\cal V}^k \]
  We define the {\em domain} of a pattern as
  $D_P := \{ (v_1,\ldots,v_k) \in {\cal V}^k \mid \exists v: (v,v_1,\ldots,v_k) \sim P \}$,
  and its {\em range} as
  $R_P := \{ v \in {\cal V} \mid \exists v_1,\ldots,v_k: (v,v_1,\ldots,v_k) \sim P\}$.
\end{definition}

From parts to whole, a pattern can be seen as a k-ary partial function
$v = P(v_1,\ldots,v_k)$. This implies that the parts contain all the
necessary information to reconstruct the whole.

From whole to parts, a pattern can be seen as a non-deterministic
partial function $(v_1,\ldots,v_k) \in P^{-1}(v)$, generating various
decompositions of the whole value into parts.

\begin{table}[t]
\centering
\begin{tabular}{|l|l|}
  \hline
  type & patterns \\
  \hline
  \hline
  {\it Grid} & {\bf BgColor}(color$:$ {\it Color},\ contents$:$ {\it Sprite}), \\
       & {\bf Monocolor}(color$:$ {\it Color}, mask$:$ {\it Mask}), \\
       & {\bf Motif}(motif$:$ {\it Motif}, core$:$ {\it Grid}, pure$:$ {\it Grid}, noise$:$ {\it Sprite}) \\
  \hline
  {\it Sprite} & {\bf IsFull}(grid$:$ {\it Grid}), {\bf Empty}(size$:$ {\it Vec}), \\
       & {\bf Monocolor}(color$:$ {\it Color}, mask$:$ {\it Mask}), \\
       & {\bf Motif}(motif$:$ {\it Motif}, core$:$ {\it Grid}, pure$:$ {\it Grid}, noise$:$ {\it Sprite}) \\
  \hline
  {\it Mask} & {\bf Empty}(size$:$ {\it Vec}), {\bf Full}(size$:$ {\it Vec}), \\
       & {\bf Point}(), {\bf Segment}(len$:$ {\it Int},\ dir$:$ {\it Vec}), \\
       & {\bf Motif}(motif$:$ {\it Motif}, core$:$ {\it Grid}, pure$:$ {\it Grid}, noise$:$ {\it Sprite}) \\
  \hline
  {\it Vector} & {\bf Vec}(i$:$ {\it Int},\ j$:$ {\it Int}), {\bf Square}(side$:$ {\it Int}) \\
  \hline
  {\it ColorMap} & {\bf Swap}(c1$:$ {\it Color},\ c2$:$ {\it Color}), {\bf Replace}(c1$:$ {\it Color},\ c2$:$ {\it Color}) \\
  \hline
\end{tabular}
\caption{Patterns by type}
\label{tab:arc:patterns}
\end{table}

\begin{arcdef}
  Table~\ref{tab:arc:patterns} lists the ARC patterns per type of the
  decomposed value, distinguishing the three types of grids. On full
  grids, ${\bf BgColor}$ decomposes the grid into a background color,
  and the rest as a sprite; ${\bf Monocolor}$ decomposes it into a
  color and a mask when a single color is present; and ${\bf Motif}$
  decomposes it according to some motif, e.g. some symmetry,
  identifying a core subgrid, and separating the pure grid following
  the motif, and some noise, possibly empty. A sprite can be
  recognized as a full grid $({\bf IsFull})$, as an empty grid with no
  colored cell $({\bf Empty})$, as having a single color
  $({\bf Monocolor})$, or as matching some motif. Similarly for masks
  plus points and segments. Points are 1x1 masks, and segments are
  decomposed into a direction (e.g., diagonal or vertical) and a
  length. Vectors can be decomposed into two integers $({\bf Vec})$,
  or recognized as square vectors where the two components are equal
  $({\bf Square})$. There are two patterns for color maps: when two
  colors replace each other $({\bf Swap})$, when a color is replaced
  by another one $({\bf Replace})$.
\end{arcdef}


\subsubsection{Expressions}
\label{expression}

Let ${\cal X}$ be a set of variables. Variables are used to identify
parts in the descriptions and models defined below. Expressions are
defined in the classical way as a combination of values, variables,
and function applications.

\begin{definition}[expression]
  An {\em expression}~$e \in {\cal E}$ is recursively defined as one of:
  \begin{itemize}
  \item $v$: a constant value from ${\cal V}$;
  \item $x$: a variable from ${\cal X}$;
  \item $f(e_1,\ldots,e_k)$: the application of a function~$f \in {\cal F}$ to $k$ arguments specified by sub-expressions.
  \end{itemize}
\end{definition}

An example of expression is $add(size(x), (1,1))$, which increase the
size of a grid~$x$ by 1 on each axis. Expressions are evaluated in the
classical way, relying on an environment to provide values for the
variables.

\begin{definition}[environment]
  An environment $\sigma \in {\cal X} \pfun {\cal V}$ is a
  partial mapping from variables to values. When a variable is
  undefined, we write $\sigma(x) = \bot$.
\end{definition}

An example of environment is a mapping
$\{ x \mapsto \begin{brsm}1 & 1 \\ 1 & 0\end{brsm}, y \mapsto 3\}$, which maps $x$ to a 2x2
grid, and $y$ to an integer.

\begin{definition}
  The evaluation~${\it eval}(e,\sigma)$ of an expression~$e$ on
  an environment~$\sigma$ returns a value or $\bot$. It is
  recursively defined as:
  \begin{itemize}
  \item ${\it eval}(v, \sigma) := v$;
  \item ${\it eval}(x, \sigma) := \sigma(x)$\\
    Note that $x$ may be undefined in $\sigma$;
  \item ${\it eval}(f(e_1,\ldots,e_k), \sigma) := f({\it eval}(e_1,\sigma), \ldots, {\it eval}(e_k, \sigma))$\\
    Note that $f$ may be undefined on its arguments. 
  \end{itemize}
\end{definition}

The evaluation of the expression $add(size(x), (1,1))$ on the above
environment returns the vector~$(3,3)$.

\subsubsection{Descriptions}
\label{description}

A description is a representation of the cascading decomposition of a
value, identifying each part with a variable.

\begin{definition}[description]
  A {\em description} $d \in {\cal D}$ of a value~$v$ is recursively defined
  as one of, where $x \in {\cal X}$ is a variable:
  \begin{itemize}
  \item $x: v$: a value (atomic description);
  \item $x: v = P(d_1,\ldots,d_k)$: a pattern-based decomposition of~$v$
    into parts~$(v_1,\ldots,v_k)$, i.e.
    $(v,v_1,\ldots,v_k) \sim P$, where each value~$v_i$ is the value
    described by~$d_i \in {\cal D}$ (composite description).
  \end{itemize}
\end{definition}

An example of description is
\[ \begin{array}{l}
     x: \begin{brsm}2 & 2 \\ 0 & 2\end{brsm} = \\
     \hspace*{1cm} BgColor(x_1: 0, x_2: \begin{brsm}2 & 2 \\ & 2\end{brsm} = \\
     \hspace*{2cm} Monocolor(x_{21}: 2, x_{22}: \begin{brsm}0 & 0 \\ & 0\end{brsm}))
   \end{array} \] where a 2x2 grid is decomposed into a background
 color (0 = black), another color (2 = red), and a mask ($\begin{brsm}0 & 0 \\ & 0\end{brsm}$).

The value described by~$d_i \in D$ is written~$v_i$ when there is no
ambiguity, and ${\it value}(d_i)$ otherwise. The root variable
of~$d_i$ is written~$x_i$ when there is no ambiguity, and
${\it var}(d_i)$ otherwise. The set of variables in a description~$d$
is written~$X_d$.
The {\em $x$-factor} of a description~$d$ is the subdescription rooted
at~$x$, noted $d.x$; and the {\em $x$-context} is the description in
which the subdescription at~$x$ is reduced to an atomic value, noted
$d.\overline{x}$.

A description provides an environment, mapping each part variable to
the associated value.  Such an environment provides access to the
whole value, to the atomic parts, and to every other intermediate values
in the description.

\begin{definition}
  Let $d$ be a description. It defines an environment~$\sigma_d$ over
  its variables~$x \in X_d$ s.t. $\sigma_d(x) = {\it value}(d.x)$.
\end{definition}

The environment defined by the above description is the mapping \[ \{ x \mapsto \begin{brsm}2 & 2 \\ 0 & 2\end{brsm}, x_1 \mapsto 0, x_2 \mapsto \begin{brsm}2 & 2 \\  & 2\end{brsm}, x_{21} \mapsto 2, x_{22} \mapsto \begin{brsm}0 & 0 \\  & 0\end{brsm} \}. \]

\subsubsection{Models}
\label{model}

Models are abstractions of descriptions, replacing some values by
unknowns and expressions.

\begin{definition}[model]
  A {\em model}~$m \in {\cal M}$ is recursively defined as one of:
  \begin{itemize}
  \item $x:\ ?$: an {\em unknown};
  \item $x: e$: an expression~$e \in {\cal E}$ that defines the value of~$x$;
  \item $x: P(m_1,\ldots,m_k)$: a pattern-based decomposition of~$x$
    into $k$ variables~$x_1,\ldots,x_k$, where each $x_i$ is modelled
    by~$m_i$.
  \end{itemize}
  We note $X_m$ the set of variables defined in a model~$m$, and the
  variable modelled by~$m_i$ is written~${\it var}(m_i)$ or
  simply~$x_i$ if there is no ambiguity.
\end{definition}

In the following, we sometimes omit the variables for the sake of
concision.
The {\em $x$-factor} of a model~$m$ is the submodel rooted at~$x$,
noted $m.x$; and the {\em $x$-context} is the model in which the
submodel at~$x$ is reduced to an unknown, noted $m.\overline{x}$.

Unlike expressions which have a single value for a given environment,
models cannot be evaluated in a deterministic way because a model
often generates many descriptions.

\begin{definition}\label{def:model:descr}
  Let $m \in {\cal M}$ be a model, $\sigma$ be an environment,
  and $d \in {\cal D}$ be a description. We say that $m$ generates $d$
  in the environment~$\sigma$, or equivalently that $d$ belongs
  to $m$ in~$\sigma$, and we note $d \in m[\sigma]$, iff the
  following statements are satisfied:
  \begin{itemize}
  \item if $m$ is $x:\ ?$, then $d$ is any atomic description $x: v$;
  \item if $m$ is $x: e$, then $d$ is the atomic description $x: v$ s.t. $v = {\it eval}(e,\sigma) \neq \bot$;
  \item if $m$ is $x: P(m_1,\ldots,m_k)$, then $d$ is the composite
    description $x: v = P(d_1,\ldots,d_k)$, s.t.
    $(v,v_1,\ldots,v_k) \sim P$, and for all $i \in 1..k$,
    $d_i \in m_i[\sigma]$.
  \end{itemize}
\end{definition}

Therefore $m[\sigma]$ denotes the set of generated
descriptions. By extension, it also denotes the set of generated
values $\{ {\it value}(d) \mid d \in m[\sigma] \}$.
For example, the model $m = x: Vec(x_1: add(x,1), x_2:\ ?)$ on the
environment $\sigma = \{x \mapsto 1\}$, generates the descriptions
$x: Vec(x_1: 2, x_2: 0)$, $x: Vec(x_1: 2, x_2: 1)$,
$x: Vec(x_1: 2, x_2: 2)$, etc.

The generation-relationship between descriptions and models can be
factorized on any variable, splitting them into factor and context.

\begin{lemma}\label{lemma:descr}
  Let $m$ be a model, $\sigma$ be an environment, $d$ be
  a description, and $x$ be a variable in~$m$ and $d$.
  \[ d \in m[\sigma] \iff d.x \in m.x[\sigma] \land d.\overline{x} \in m.\overline{x}[\sigma] \]
\end{lemma}

\begin{proof}
  We first prove the forward implication.  The relationship
  $d \in m[\sigma]$ boils down to a set of constraints,
  $v = {\it eval}(e,\sigma)$ for each expression and
  $(v,v_1, \ldots, v_k) \sim P$ for each pattern, with a one-to-one
  correspondence between the tree structures of $d$ and $m$.
  Therefore, the split of trees at $x$ between the subtrees ($d.x$ and
  $m.x$) and the contexts ($d.\overline{x})$ and $m.\overline{x}$)
  entails a partition of the set of constraints in two parts, one for
  the subtrees and another for the contexts.
  As a consequence, the relationship holds for the subtrees
  ($d.x \in m.x[\sigma]$) and for the contexts
  ($d.\overline{x} \in m.\overline{x}[\sigma]$).

  Because the two subsets of constraints form a partition of the whole
  set of constraints, the argument can be reversed, hence proving the
  backward implication.
\end{proof}

\subsubsection{Tasks and Task Models}
\label{task}

We can now define tasks as sets of input-output pairs, and task models
that act as solutions to a task. We assume from the domain a
distribution~$V^i$ of input values, and a distribution~$V^o$ of output
values. In ARC, those distributions are defined over 10-color grids
with size at most 30x30.

\begin{definition}[task]
  A {\em task} is a structure $T = \langle E, F \rangle$, where:
  $E, F \subseteq {\cal V} \times {\cal V}$ are two sets of {\em
    examples} $(v^i,v^o)$, pairs of input and output, s.t.
  $v^i \sim V^i$ and $v^o \sim V^o$.
  The working assumption is that
  there exists a program mapping all inputs to their output. The
  examples are split into two subsets: the training examples~$E$ and
  the test examples~$F$.
\end{definition}

An example task is shown in Figure~\ref{fig:example}. It has three
training examples and one test example.

\begin{definition}[task model]
  A {\em task model} is a pair
  $M = (m^i,m^o) \in {\cal M} \times {\cal M}$ s.t. all variables used
  in~$m^o$ are introduced in~$m^i$: $m^i$ is called the {\em input
    model}, and $m^o$ the {\em output model}.
\end{definition}

Figure~\ref{fig:example:model} shows an example of task model.
We also extend the definition of descriptions to examples, i.e. to
pairs of values.

\begin{definition}[example description]
  An {\em example description} is a pair of descriptions, one for the
  input and the other for the output:
  $D = (d^i,d^o) \in {\cal D} \times {\cal D}$.
\end{definition}

We now extend the generation-relationship to example descriptions and
task models. Note that the output description depends on the input
description. Indeed, the final objective is to be able to predict the
output from the input.

\begin{definition}
  A task model~$M = (m^i,m^o)$ generates an example
  description~$D = (d^i,d^o)$ iff the input model generates the input
  description, and the output model generates the output description
  using the input description as environment, i.e.
  \[ D \in M \iff d^i \in m^i[\emptyset], d^o \in m^o[\sigma_{d^i}]. \]

  A task model generates an example iff it generates a
  task description for those values.
  \[ (v^i,v^o) \in M \iff \exists D = (d^i,d^o) \in M: {\it value}(d^i) = v^i, {\it value}(d^o) = v^o \]
\end{definition}

\HIDE{
\begin{definition}[definite task model]
  A task model~$M = (m^i,m^o)$ is said {\em definite} if the output
  model contains no unknown, i.e. only patterns and expressions.
\end{definition}

In a definite model, the output model generates a single value, when
defined, given an environment. This implies that for a given input
description, there is a single output description, hence a single
output value. However, for a given input, there may be several input
descriptions, hence several output descriptions and values.
}

\subsection{Description Lengths}
\label{dl}

As said above, the MDL principle states that the best model is the
model that compresses the data the more. The data is here the set of
training examples. The a priori description length of the training
examples~$E$ of a task $T = \langle E, F \rangle$ can be defined as
follows.
\[ L(E) := \sum_{(v^i,v^o) \in E} L(v^i \mid V^i) + L(v^o \mid V^o) \]
Each value is encoded according to the value distributions of the
domain. Those value encodings are simply concatenated because the
examples are assumed independent, and no relationship is known between
the input and ouput values (this is what we want to learn).

\begin{arcdef}
  For ARC, the input and output value distributions are both defined
  as a distribution~$V_g$ of 10-color grids s.t. for any $h \times w$
  grid~$g$, we have
  \[ L(g \mid V_g) = L_{\mathbb{N}^*}(h) + L_{\mathbb{N}^*}(w) + hw
    \log_2 10. \] A grid is encoded by concatenating the codes for the
  height and width of the grid, and the codes for the color of each
  cell of the grid.
\end{arcdef}

According to two-part MDL, the information contained in examples can
be split between a task model and their description according to the
model.
\[ L(M, E) := L(M) + L(E \mid M) \]
%
The MDL principle tells that the best model~$M^*$ is the one that compresses
the more the data.
\[ M^* := \argmin_{M} L(M, E) \]

We now have to define the two parts: model and data description
lengths. We start with data.

\subsubsection{Encoding Examples}
\label{dl:data}

Encoding a set of examples amounts to encode each example.  Encoding
an example~$(v^i,v^o)$ according to a model~$M$ amounts to encode an
example description~$D$ of those values. A difficulty is that there
may be several descriptions for an example. Along the MDL principle,
we choose the most compressive description.

\[ D^*_M(v^i,v^o) := \argmin_{\substack{D \in M\\{\it value}(D) = (v^i,v^o)}} L(D \mid M) \]

We can now define the description length of data as follows.

\[ L(E \mid M) := \sum_{(v^i,v^o) \in E} L(D^*_M(v^i,v^o) \mid M) \]

Encoding an example description can be decomposed in two parts: the
input and the output.

\[ L(D \mid M) := L(d^i \mid V^i, m^i[\emptyset]) + L(d^o \mid V^o, m^o[\sigma_{d^i}]) \]

Therefore, it remains to define $L(d \mid V, m[\sigma])$, the
description length of a value drawn from a distribution~$V$ when
decomposed into a description~$d$ that belongs to the model~$m$
applied to the given environment~$\sigma$. This is done by
induction on the tree structure common to the description and the
model.
\begin{itemize}
\item if $d$ is $(x: v)$ and $m$ is $(x:\ ?)$, then $L(d \mid V, m[\sigma]) := L(v \mid V)$;
\item if $d$ is $(x: v)$ and $m$ is $(x: e)$, then $L(d \mid V, m[\sigma]) := 0$;
\item if $d$ is $(x: v = P(d_1,\ldots,d_k))$ and $m$ is $(x: P(m_1,\ldots,m_k))$, then
  \[ L(d \mid V, m[\sigma]) := \sum_{i=1}^k L(d_i \mid V_i,
    m_i[\sigma]), \] where $V_i$ is the value distribution of the
  part~$x_i$ of pattern~$P$ given that the whole value~$x$ is drawn
  from~$V$, and that the value of all parts~$x_{j<i}$ is known. This
  is to account for known constraints on the whole value, and
  dependencies between the parts of a pattern.
\end{itemize}

Encoding an unknown value is simply encoding the value, according to
its distribution. Encoding an expression value is not necessary
because the value can be computed from the expression and the provided
environment. Encoding a composite value can be reduced to encoding its
parts because the whole value can be computed deterministically from
the values of the parts.
The specification of $V_i$ comes from the chain decomposition of the
encoding of the row of parts~$d_1,\ldots,d_k$. It leads to the
following requirement about patterns.

\begin{require}[distributions of pattern parts]
  For every pattern~$P \in {\cal P}$, and every part position~$i$,
  define the distribution~$V_i = V_{P,i}(V,v_1,\ldots,v_{i-1})$, as a
  function of the whole value distribution~$V$, and the value of
  previous parts~$v_1,\ldots,v_{i-1}$.
\end{require}

\subsubsection{Encoding Models}
\label{dl:model}

Encoding a task model can be decomposed in two parts, the input model
and the output model.

\[ L(M) := L(m^i \mid V^i) + L(m^o \mid V^o, X_{m^i}) \]

Note that encoding the output model depends on the variables defined
by the input model because the output model can refer to them through
expressions. We therefore have to define $L(m \mid V, X)$, the
description length of a model given the distribution of the values to
be modeled, and a set of available variables (that can be used in
expressions). For the input model, $X = \emptyset$.

There is no standard way to encode models. We propose to follow the
idea that smaller models are prefered. We hence decompose the encoding
of a model into (1) the encoding of its size, (2) the encoding of its
abstract syntax given its size, and finally (3) the encoding of the
remaining elements: constant values and the choice of variables.

The size~$n$ of a model~$m$ is defined as the number of its symbols,
i.e. unknowns, patterns, functions, variables, and values.  For
example, the size of $?$ is 1; the size of $Vec(add(x,1), ?)$ is 5.
The encoding of this size is defined as $L_{\mathbb{N}^*}(n)$, the
Elias encoding of positive integers.


We can compute for each size~$n > 0$ and each value distribution~$V$ the
number of models $\#M(n,V)$ and the number of expressions $\#E(n,V)$
having that size (without actually enumerating them).
\begin{align}
\#M(1,V) &= 1 + \#E(1,V) \\ 
\#M(n>1,V) &= \#E(n,V) + \sum_{\substack{P \in {\cal P}\\ R_P \cap V \neq \emptyset}} \sum_{\substack{P \textrm{ has arity } k\\n_1 > 0, \ldots, n_k > 0\\n_1+\ldots+n_k = n-1}} \prod_{i=1}^k \#M(n_i,V_{P,i}(V)) \\
\#E(1,V) &= 2 \\ 
\#E(n>1,V) &= \sum_{\substack{f \in {\cal F}\\R_f \cap V \neq \emptyset}} \sum_{\substack{P \textrm{ has arity } k\\n_1 > 0, \ldots, n_k > 0\\n_1+\ldots+n_k = n-1}} \prod_{i=1}^k \#E(n_i, V_{f,i}(V))
\end{align}
Equation~(1) says that atomic models are the unknowns and the atomic
expressions. Equation~(2) says that a $(n>1)$-size model is either a
$n$-size expression or a pattern with a $V$-compatible range, and with
the remaining size~$(n-1)$ distributed over the $k$ parts. Equation~(3)
says that an atomic expression is either a constant value or a
variable. Equation~(4) says that a $(n>1)$-size expression is a function
defined with a $V$-compatible range, and with the remaining size~$(n-1)$
distributed over the $k$ arguments. Those definitions entail the
following requirement about patterns and functions.

\begin{require}(distributions of pattern parts and function arguments)
  For every pattern $P \in {\cal P}$, and every part position~$i$,
  define the part value distribution $V_{P,i}(V)$ as a function
  of the whole value distribution~$V$.
  Similarly, for every function $f \in {\cal F}$, and every argument
  position~$i$, define the argument value distribution $V_{f,i}(V)$ as
  a function of the result value distribution~$V$.
\end{require}

Those cascading distributions can be seen as a generative grammar of
models and expressions, where distributions play the role of
non-terminals and where each pattern or a function play the role of a
production rule. Table~\ref{tab:arc:patterns} provides such a grammar,
simplifying distributions into types. For instance, we have
$V_{{\bf BgColor},1}({\it Grid}) = {\it Color}$ and
$V_{{\bf BgColor},2}({\it Grid}) = {\it Sprite}$.
We assume a uniform distribution over the models of same size, so that
the description length of the abstract syntax of a model given its
size~$n$ is defined as $\log_2 \#M(n,V)$.

It remains to encode the values and variables as a function of the
submodel value distribution~$V_y$, computed recursively through the
submodel context via $V_{P,i}$ and $V_{f,i}$. For a submodel $y: v$,
the value is encoded in $L(v \mid V_y)$ bits. For a submodel $y: x$,
the variable is encoded according to the uniform distribution over the
subset of variables in~$X$ that are compatible with~$V_y$, i.e.
$V_x \cap V_y \neq \emptyset$ where $V_x$ is the local distribution of
$x$ in the input model.


\subsection{Learning Goal and Search Space}
\label{search:space}

Given a task $T = \langle E, F \rangle$, task models can be
used under two inference modes:
\begin{itemize}
\item {\em Description} of a training example~$(v^i,v^o) \in E$, where
  both input and output values are known.
  \[ D^*_M(v^i,v^o) := \argmin_{\substack{D \in
        M\\{\it value}(D) = (v^i,v^o)}} L(D \mid M) \]
  This is the most compressive description compatible with the example, as generated by
  the task model.
\item {\em Prediction} of the output value given an input value,
  written as the application of the model to the input value.
  \begin{align*}
    M(v^i) & := {\it value}(d^{o*}) \\
           & \textrm{where } D^* = (d^{i*},d^{o*}) = \argmin_{\substack{D = (d^i,d^o) \in M\\ {\it value}(d^i) = v^i}} L(D \mid M)
    \end{align*}
  This is the output value of the most compressive description compatible with the input value, as generated by the task model.
\end{itemize}
In case there are no compatible description, $\bot$ is returned as the
undefined result. In case several descriptions have the same minimal
description length, the choice is left unspecified. In practice, it
can simply be specified through a total ordering on values.

A task model is a {\em solution} to a task if it correctly predicts the
output for all training examples. It is said to {\em generalize} if it is
also correct on all test examples.

\begin{definition}[correct task model]
  A task model~$M$ is said {\em correct} on an example~$(v^i,v^o)$
  iff $M(v^i) = v^o$.
\end{definition}

\begin{definition}[solution and generalization]
  A task model~$M$ is said to be a {\em solution} on task~$T$ if it is
  correct on all training examples~$E$. This solution is said to {\em
    generalize} if it is also correct on the test examples~$F$.
\end{definition}


We define a search space over all task models by defining an initial
task model and transitions between models so as to make all models
reachable. We note $M.x$ the submodel of $M$ at variable~$x$. We note
$M[x \leftarrow m']$ the substitution of $m'$ to $M.x$ in $M$
s.t. ${\it var}(m') = x$ (to avoid variable renaming) and
$X_{m'} \cap X_M = \{x\}$ (to avoid variable capture).

\begin{definition}[search space]
  The search space over task models is defined as
  $\langle {\cal S}, M_0, \Delta \rangle$ where:
  \begin{itemize}
  \item ${\cal S}$ is the set of states, here task models;
  \item $M_0 = (x_1:\ ?, x_2:\ ?)$ is the initial model, defined as
    the most unspecific task model that generates all pairs of values in $V^i \times V^o$;
  \item
    $\Delta \subseteq \{(M,x,m',M') \mid M, M' \in {\cal S}, x \in X_M, m' \in {\cal
      M}, M' = M[x \leftarrow m'] \}$ is a set of transitions from a
    task model~$M$ to a task model~$M'$ that results from the
    replacement of the submodel~$M.x$ by~$m'$. The resulting
    model~$M'$ must be well-formed, i.e. all variables used in~$m^o$
    remain defined in~$m^i$.
  \end{itemize}
\end{definition}

We first prove that this search space is complete -- in the sense that
all task models are reachable -- given a set of minimal
transitions. This ensures that solutions can be found by traversing a
finite number of transitions.

\begin{theorem}
  Every task model can be reached through a finite number of
  transitions by considering only two classes of {\em minimal
    transitions}:
  \begin{enumerate}
  \item $x \leftarrow P$: given the pattern $P \in {\cal P}$ of arity~$k$,
    for any model~$M$ where $M.x = x:\ ?$, the transition
    $(M,x,P(x_1:\ ?, \ldots, x_k:\ ?),M')$ where $x_1, \ldots, x_k$
    are fresh variables;
  \item $x \leftarrow e$: given the well-defined expression~$e$, for
    any model~$M$ where $M.x = x:\ ?$, the transition $(M,x,e,M')$.
  \end{enumerate}
\end{theorem}

\begin{proof}
  By recurrence on the size of the target task model~$M'$.  If there
  is a $x \in X_{M'}$ s.t. $M'.x = e$, then there is a strictly
  smaller task model $M = M'.\overline{x}$ s.t. $(M, x, e, M')$ is a
  minimal transition, instance of $x \leftarrow e$. By recurrence
  hypothesis, $M$ is reachable, hence $M'$ is reachable.

  Otherwise, if there is $x \in X_{M'}$ s.t.
  $M'.x = P(x_1:\ ?, \ldots, x_k:\ ?)$, then there is a strictly
  smaller task model $M = M'.\overline{x}$ s.t.
  $(M,x, P(x_1:\ ?, \ldots, x_k:\ ?), M')$ is a minimal transition,
  instance of $x \leftarrow P$. By recurrence hypothesis, $M$ is
  reachable, hence $M'$ is reachable.

  Otherwise, $M'$ must be the initial model~$M_0$ because any model is
  one of a pattern, an expression or an unknown.
\end{proof}

Note that exprression transitions cannot be decomposed into
transitions that would introduce values, variables and functions one
at a time because functions can only be evaluated when all their
arguments are defined. This is the key difference with patterns, and
the key benefit of patterns that can be introduced in the model
piecewise.

The set of minimal transitions leading from~$M_0$ to any task
model~$M$ is uniquely defined as follows, introducing each pattern and
each expression in~$M$.
\[
  \begin{array}{lll}
    \Delta(M) & = & \{ x \leftarrow P \mid x \in X_M, M.x = P(m_1, \ldots, m_k) \} \\
              & \cup & \{ x \leftarrow e \mid x \in X_M, M.x = e \}
  \end{array}
\]
We can decompose it into $\Delta(M) = \Delta^i(M) \uplus \Delta^o(M)$,
where the first term is the subset of transition on the input model
($x \in X_{m^i}$), and the second term is the subset of transitions on
the output model ($x \in X_{m^o}$).

The order in which transitions can be applied is constrained. A
transition $x \leftarrow m'$ can only be triggered after variable~$x$
has been introduced. Moreover, an expression transition
$x \leftarrow e$ can only be triggered after all variables used by the
expression have been introduced. This results in a partial ordering
over $\Delta(M)$. In particular, input model transitions do not depend
on output model transitions; and ouput pattern transitions do not
depend on input transitions. Also, an output expression transition
depends on both input and output pattern transitions but not on other
expression transitions.
Every total ordering compatible with this partial ordering is a valid
path to the target task model~$M$. There are two simple strategies:
all input transitions before all output transitions, and all pattern
transitions before all expression transitions. But the most effective
search strategy could be a more mixed version.

In order to guide the search towards a solution, it is beneficial to
identify a property of task models that remains valid all throughout a
path from the initial state to the solution. Indeed, this may allow to
prune vast portions of the search space. This can be achieved by
relaxing the constraint to predict the expected output into the
constraint to find a description of the input-output pair.

\begin{definition}[consistent task model]
  A task model $M$ is said {\em consistent} with the task~$T$ --
  written $T \vdash M$ -- if it can describe all training examples,
  i.e.: 
  $D^*_M(v^i,v^o) \neq \bot$, for all $(v^i,v^o) \in E$.
\end{definition}

Optionally, consistency can be extended to test inputs by stating that
they must have a description by the input model
($v^i \in m^i[\emptyset]$).
We prove in the following that consistency can be used as a pruning
property when searching through minimal transitions. We start by
proving a lemma on models (not task models).

\begin{lemma}
  Let $m, m'$ be two models s.t. $m'$ results from the traversal of a
  minimal transition $x \leftarrow P$ or $x \leftarrow e$ from $m$,
  for some pattern~$P$ or expression~$e$, and environment
  $\sigma$. For every description generated by~$m'$, there is a
  description generated by~$m$.
  \[ d' \in m'[\sigma] \Longleftarrow d = d'.\overline{x} \in m[\sigma] \]
\end{lemma}

\begin{proof}
  This is a direct consequence of Lemma~\ref{lemma:descr} because,
  from the definition of minimal transitions, we have
  $m'.\overline{x} = m$. Indeed, the original submodel at $x$ is the
  unknown.
\end{proof}

We can now prove the anti-monotony of consistency: if a task model is
inconsistent, then all task models reachable from it through minimal
transitions will also be inconsistent.

\begin{theorem}\label{theo:consistency}
  Let $T$ be a task. For every task model $M_2$ that results from the
  traversal of a minimal transition $x \leftarrow m'$ from a task
  model~$M_1$, we have that if $M_1$ is not consistent with the task, then
  $M_2$ is not consistent either: $T \nvdash M_1 \Rightarrow T \nvdash M_2$.
\end{theorem}

\begin{proof}
  We prove the contrapositive, assuming $T \vdash M_2$ and proving
  that $T \vdash M_1$.  First, we can observe that, by definition of a
  minimal transition, $M_2.\overline{x} = M_1$. Without loss of
  generality, we assume that $x$ belongs to the output model, so that
  $m_2^o.\overline{x} = m_1^o$, and $m_2^i = m_1^i$.

  We have $T \vdash M_2$\\
  $\Rightarrow \forall (v^i,v^o) \in E: D^*_{M_2}(v^i,v^o) \neq \bot$\\
  $\Rightarrow \forall (v^i,v^o) \in E: \exists (d_2^i,d_2^o) \in M_2: {\it value}(d_2^i) = v^i, {\it value}(d_2^o) = v^o$\\
  $\Rightarrow \forall (v^i,v^o) \in E: \exists (d_2^i,d_2^o): d_2^i \in m_2^i[\emptyset], d_2^o \in m_2^o[\sigma_{d_2^i}]$\\
  $\Rightarrow \forall (v^i,v^o) \in E: \exists (d_2^i,d_2^o): d_2^i \in m_2^i[\emptyset], d_2^o.\overline{x} \in m_2^o.\overline{x}[\sigma_{d_2^i}]$ (Lemma~\ref{lemma:descr})\\
  $\Rightarrow \forall (v^i,v^o) \in E: \exists (d_2^i,d_2^o): d_2^i \in m_1^i[\emptyset], d_2^o.\overline{x} \in m_1^o[\sigma_{d_2^i}]$ (see above observation)\\
  $\Rightarrow \forall (v^i,v^o) \in E: \exists (d_2^i,d_2^o): (d_2^i, d_2^o.\overline{x}) \in M_1$ (by definition of $D \in M$)\\
  $\Rightarrow D^*_{M_1}(v^i,v^o) \neq \bot$, from the values of $d_2^i$ and $d_2^o$ above\\
  Hence $T \vdash M_1$.  
\end{proof}

\HIDE{
As a corollary, the minimal transitions of a model~$M$ to consider are
only the following ones:
\begin{itemize}
\item $x \leftarrow P$, when $\forall (v^i,v^o) \in E: \exists D \in M:  D^*_M(v^i,v^o): P^{-1}({\it value}(D.x)) \neq \emptyset$
\item $x \leftarrow e$, when
  $\forall (v^i,v^o) \in E: \exists D \in D^*_M(v^i,v^o): {\it eval}(e,\sigma) = {\it value}(D.x)$,
  where the environment is $\sigma = \emptyset$ when $x$ is in
  the input model, and $\sigma = \sigma_{d^i}$ when $x$ is in
  the output model.
\end{itemize}
}

We note $\Delta(M \mid T) \subseteq \Delta(M)$ the subset of minimal
transitions of a task model that are consistent with the task. We can
also write $\Delta^i_P(M \mid T)$ to restrict to consistent
pattern-transitions on the input model, or $\Delta^o_e(M \mid T)$ to
restrict to consistent expression-transitions on the output model. We
can also write $\Delta_x(M \mid T)$ to restrict to consistent
transitions on variable~$x \in X_M$.

Although consistency enables to prune out patterns and expressions
that do not agree with the training examples, it may still remain
multiple consistent transitions at each step, entailing an exponential
growth of reachable states. It is therefore desirable to also have a
heuristic to guide the search towards the more promising region. We
use the description length~$L(M,D)$ defined in the previous section
along the MDL principle: the best models are those that compress the
data the more. Note that, although description length tends to
decrease with more specific models, a more specific model can have a
larger description length if the inserted pattern is not a good model
of the examples.

\section{Algorithms and Pragmatic Aspects}
\label{algos}

In this section, we detail the key algorithms of the MADIL approach,
from model-based parsing to search for a task model that solves a
task. Those algorithms are generic, and we point at the parts that
depend on the task domain. We discuss the few parameters that help
control the search for parses and the search for task models. We also
discuss a few pragmatic aspects, i.e. small deviations from the theory
and algorithms that are motivated by performance issues.

\subsection{Parsing and Generation: Finding Good Descriptions}
\label{algos:parsing}
\label{algos:generation}


\begin{algorithm}[t]
  \caption{Parsing a value into descriptions with a model}
  \label{algo:parse}
  \begin{algorithmic}[1]
    \Require $V$: a distribution of values,
    $m$: a model,
    $\sigma$: an environment
    \Require $v \sim V$: the value to be parsed
    \Ensure a lazy sequence of pairs $(d,l)$ s.t. $d \in m[\sigma]$, ${\it value}(d) = v$, and \mbox{$l = L(d \mid V, m[\sigma])$}
    \Function{parse}{$V, m, \sigma, v$}
    \If{$m$ {\bf like} $x: \texttt{?}$}
      \State \textbf{yield} $(x: v, L(v \mid V))$
    \ElsIf{$m$ {\bf like} $x: e$}
      \If{$\text{eval}(e, \sigma) = v$}
        \State \textbf{yield} $(x: v, 0)$
      \EndIf
    \ElsIf{$m$ {\bf like} $x: P(m_1, \ldots, m_k)$}
        \ForAll{$(v_1, \ldots, v_k) \in P^{-1}(v)$}
        \ForAll{$(d_1, l_1) \in \text{\sc parse}(V_{P,1}(V), m_1, \sigma, v_1)$}
        \State \quad\vdots
        \ForAll{$(d_k, l_k) \in \text{\sc parse}(V_{P,k}(V,v_1, \ldots, v_{k-1}), m_k, \sigma, v_k)$}
        \State \textbf{yield} $(x: v = P(d_1, \ldots, d_k), l_1 + \ldots + l_k)$
        \EndFor
        \EndFor
        \EndFor
    \EndIf
    \EndFunction    
  \end{algorithmic}
\end{algorithm}

Algorithm~\ref{algo:parse} defines the {\sc parse} function that
outputs a lazy sequence of all descriptions -- along with their
description length -- that belong to a model~$m$ and that have some
fixed value~$v$, under some environment~$\sigma$. It is defined
recursively by induction on the model syntax tree, along
Definition~\ref{def:model:descr} and Section~\ref{dl:data}. As the
sequence is computed lazily, it is possible to efficiently compute a
limited number of descriptions.
The parts of the algorithm that are specific to a domain are:
$L(v \mid V)$, the description length of a value; $P^{-1}(v)$, the
pattern-based decomposition of a value; and
$V_{P,i}(V,v_1,\ldots,v_{i-1})$, the value distributions of each part
of a pattern.
It is not possible in general to generate descriptions in increasing
DL order.  However, the order of generation can be improved in two
ways.
First, $P^{-1}$ can be designed so as to generate the more
promising decompositions first. For example, with the ${\bf Motif}$
pattern, a smaller periodic pattern is more promising than a larger
one because it reduces a grid to a smaller subgrid.
Second, for the $k$ nested loops in lines 9-11 that perform a
Cartesian product of the $k$ sequences of subdescriptions for each
part, a sorted Cartesian product by rank can be used to favor low-rank
descriptions.

\begin{algorithm}[t]
  \caption{Describing a pair of values with a task model}
  \label{algo:describe}
  \begin{algorithmic}[1]
    \Require $M = (m^i, m^o)$: a task model
    \Require $v^i \sim V^i$, $v^o \sim V^o$: the input and output values to be described
    \Ensure a list of pairs $(D, L)$ s.t. $D \in M$ and $L = L(D \mid M)$, in ascending $L$-order
    \Function{describe}{$M,v^i,v^o$}
    \State $S \gets \emptyset$
    \ForAll{$(d^i,l^i) \in \textsc{top}_{K_p}(\textsc{sample}_{N_p}(\textsc{parse}(V^i,m^i,\emptyset,v^i)))$}
    \ForAll{$(d^o,l^o) \in \textsc{top}_{K_p}(\textsc{sample}_{N_p}(\textsc{parse}(V^o,m^o,\sigma_{d^i},v^o)))$}
    \State $D \gets (d^i,d^o)$
    \State $L \gets l^i + l^o$
    \State $S \gets S \cup \{(D,L)\}$
    \EndFor
    \EndFor
    \State \textbf{return} $S$ sorted by ascending $L$
    \EndFunction
  \end{algorithmic}
\end{algorithm}

Algorithm~\ref{algo:describe} builds on the {\sc parse} function to
define the {\sc describe} function that outputs a list of example
descriptions~$D$ that belong to a task model~$M$ for a fixed pair of
input-output values, in increasing DL order. For each description of
the input, it computes descriptions of the output, and sum their
description lengths.
For tractability reason, we approximate the result by sampling $N_p$
descriptions and selecting the top $K_p$ descriptions (smaller DL),
both for input and output. The maximum number of returned example
descriptions is therefore $K_p^2$. Typical parameter values are
$N_p=100$ and $K_p=3$, we study the impact of those parameters in the
evaluation section (Section~\ref{eval}).

\begin{algorithm}[t]
  \caption{Generating descriptions with a model}
  \label{algo:generate}
  \begin{algorithmic}[1]
    \Require $V$: a distribution of values, $m$: a model, $\sigma$: an environment
    \Ensure a lazy sequence of pairs $(d,l)$ s.t. $d \in m[\sigma]$ and $l = L(d \mid V, m[\sigma])$
    \Function{generate}{$V, m, \sigma$}
    \If{$m$ {\bf like} $x: \texttt{?}$}
      \ForAll{$v \sim V$}
        \State \textbf{yield} $(x: v, L(v \mid V))$
      \EndFor
    \ElsIf{$m$ {\bf like} $x: e$}
      \State{$v \gets \text{eval}(e, \sigma)$}
      \State \textbf{yield} $(x: v, 0)$
    \ElsIf{$m$ {\bf like} $x: P(m_1, \ldots, m_k)$}
        \ForAll{$(d_1, l_1) \in \text{\sc generate}(V_{P,1}(V), m_1, \sigma)$}
        \State \quad\vdots
        \ForAll{$(d_k, l_k) \in \text{\sc generate}(V_{P,k}(V,v_1, \ldots, v_{k-1}), m_k, \sigma)$}
        \State $v \gets P(v_1, \ldots, v_k)$
        \State \textbf{yield} $(x: v = P(d_1, \ldots, d_k), l_1 + \ldots + l_k)$
        \EndFor
        \EndFor
    \EndIf
    \EndFunction    
  \end{algorithmic}
\end{algorithm}

Algorithm~\ref{algo:generate} defines the {\sc generate} function that
outputs a lazy sequence of all descriptions that belong to a model~$m$,
under some environment~$\sigma$. It is defined similarly to function
{\sc parse}, except that values are drawn from distributions or
expression evaluation, and composed by patterns. Note that expression
evaluation (line~6) and pattern composition (line~12) may be undefined
for some parameters, in which case there is no yield.
The parts of the algorithm that are specific to a domain are:
$v \sim V$, drawing values from a distribution; $P(v_1,\ldots,v_k)$,
pattern-based composition; and $V_{P,i}(V,v_1,\ldots,v_{i-1})$, the
value distributions of each part of a pattern.
It is not possible in general to generate descriptions in increasing
DL order.  However, the order of generation can be improved similarly
to function {\sc parse} by using a sorted Cartesian product by rank
for the $k$ nested loops in lines 9-11.
%

\begin{algorithm}[t]
  \caption{Predicting outputs from an input with a task model}
  \label{algo:predict}
  \begin{algorithmic}[1]
    \Require $M = (m^i, m^o)$: a task model
    \Require $v^i \sim V^i$: the input value to be described
    \Ensure a list of pairs $(v^o,L)$ in ascending $L$-order
    \Function{predict}{$M,v^i$}
    \State $S \gets \emptyset$
    \ForAll{$(d^i,l^i) \in \textsc{top}_{K_p}(\textsc{sample}_{N_p}(\textsc{parse}(V^i,m^i,\emptyset,v^i)))$}
    \ForAll{$(d^o,l^o) \in \textsc{top}_{K_g}(\textsc{sample}_{N_g}(\textsc{generate}(V^o,m^o,\sigma_{d^i})))$}
    \State $v^o \gets \text{value}(d^o)$
    \State $L \gets l^i + l^o$
    \State $S \gets S \cup \{(v^o, L)\}$
    \EndFor
    \EndFor
    \State \textbf{return} $S$ sorted by ascending $L$, minus duplicate values
    \EndFunction
  \end{algorithmic}
\end{algorithm}

Algorithm~\ref{algo:predict} builds on the {\sc parse} and {\sc
  generate} functions to define the {\sc predict} function that
outputs a list of predicted outputs~$v^o$ given a task model~$M$ and a
fixed input~$v^i$, in increasing DL order. For each description of the
input, it generates output descriptions, hence an output value, and
sum their description lengths.
For tractability reason, we approximate the result by sampling $N_p$
input descriptions, selecting the top $K_p$ input descriptions among
them, and finally by sampling $N_g$ output descriptions and selecting
the top~$K_g$ for each selected input description. The maximum number
of predicted outputs is therefore $K_pK_g$, although only the first
few are used in benchmarks (3 in ARCathon, 2 in ArcPrize). Typical
parameter values are $N_p=100$, $K_p=3$, and $N_g=K_g=3$, we study the
impact of those parameters in the evaluation section
(Section~\ref{eval}).

\subsection{Transitions: Finding Promising Model Refinements}
\label{algos:transitions}

\begin{algorithm}[t]
  \caption{Consistent transitions from a given task model}
  \label{algo:transitions}
  \begin{algorithmic}[1]
    \Require $M$ a task model
    \Require $\{D_{ij}\}$ a set of example descriptions, indexed by example ID $i$, and by parsing rank $j$
    \Ensure $\Delta$: a list of transitions $(M,x,m',M')$ from $M$, consistent with $\{D_{ij}\}$, in ascending $L$-order
    \Function{transitions}{$M, \{D_{ij}\}$}
    \State $\Delta \gets \emptyset$
    \ForAll{$x \in X_M$ s.t. $M.x$ is not an expression}
      \State $\{\sigma_{ij}\} \gets \{\sigma_{d^i} \textrm{~if~} x \in X_{m^o} \textrm{~else~} \emptyset \mid D_{ij} = (d^i,d^o)\}$ 
      \State $\{L_{ij}\} \gets \{L(M, D_{ij})\}$
      \ForAll{$(m',L) \in \bigcap_{i,+} \bigcup_{j,{\it min}} \textsc{submodels}(M.x, \sigma_{ij}, D_{ij}.x, L_{ij})$}
        \State $M' \gets M[x \gets m']$
        \State $L \gets L - L(M.x) + L(m')$
        \State $\Delta \gets \Delta \cup \{ ((M,x,m',M'), L) \}$
      \EndFor
    \EndFor
    \State \textbf{return} $\Delta$ sorted by ascending $L$
    \EndFunction

    \Function{submodels}{$m, \sigma, d, L$}
    \State $v, V \gets {\it value}(d), {\it distrib}(d)$
    \State $l \gets L(d \mid V, m[\sigma])$
    \State ${\cal E} \gets$ a finite collection of expressions evaluated over $\sigma$
    \State ${\cal M}_P \gets$ a finite collection of pattern-based sub-models for $V$-values
    \State \textbf{return} $\{(v, L-l)\}$
    \State \hspace*{1.3cm} $\cup\ \{(e, L-l) \mid e \in {\cal E}, {\it eval}(e,\sigma) = v\}$
    \State \hspace*{1.3cm} $\cup\ \{(m', L-l+l') \mid m' \in {\cal M}_P, \textsc{parse}(V,m',\sigma,v) = (d',l'), \ldots \}$
    \EndFunction
  \end{algorithmic}
\end{algorithm}

Algorithm~\ref{algo:transitions} defines function {\sc transitions}
that returns a list of consistent transitions starting from a task
model~$M$, given a list of lists of example descriptions, where each
sublist is generated by function~{\sc describe} on a training example.
The returned transitions include the minimal transitions defined above
but also consider refining a pattern by a more specific pattern or an
expression. That is why transitions are computed for every submodel
that is not an expression (line~3), i.e. is an unknown or a
pattern. This makes search less dependent on the exact ordering of
transitions.
For a given submodel~$M.x$, candidate transition submodels~$m'$ are
computed for each example description~$D_{ij}$ through function~{\sc
  submodels}, and then aggregated by union over the different
descriptions~$D \in \{D_{ij}\}$ of an example, and finally by
intersection over all examples (line~6). The rationale is that a
consistent submodel must be consistent for {\em some} description of
{\em every} example. Function~{\sc submodels} also take as input the
environment relative to the description~$D_{ij}$ and to the location
of the submodel, either in the output or in the input (line~4); and
the DL of that description relative to the current model~$M$ (line~5).
Each submodel~$m'$ comes with an estimate
DL~$L = L(M,D) - L(D.x \mid M.x) + L(D.x \mid m')$ that takes into
account the replacement of~$M.x$ by~$m'$ in the encoding of the
description.
Therefore, {\em minimum} is used to combine the DLs of the different
descriptions of an example, and {\em addition} is used to combine the
DLs of the different examples.
Lines~7-9 computes the target model~$M'$ of the transition, and the estimate DL
combining model change and description change.

Function~{\sc submodels} return three kinds of candidate
submodels~$m'$ for a given submodel~$m$ and the value~$v$ of its
corresponding sub-description~$d$: the constant value~$v$,
expressions~$e$, and pattern-based submodels~$P(\ldots)$.
Expressions and their values are taken from a finite
collection~${\cal E}$ derived from the environment~$\sigma$. For
instance, that collection could be all well-typed expressions composed
of up to 6 $\sigma$-variables, constants and function calls. For
efficiency, expressions are indexed by their value into a DAG data
structure~\cite{FlashFill2013}. This offers a compact representation
of a large set of expressions, and the quick retrieval of all
expressions that evaluate to some value, here~$v$.
Pattern-based models are taken from a finite collection~${\cal M}_P$
derived from the type of value~$v$. This collection should at least
contain models like $P(?, \ldots, ?)$ for all relevant patterns, and
it may also contain models composed of several patterns and constant
values, as a shorthand for common sequences of transitions. This may
also increase the chance to find a compressive transition. Only
pattern-based models that can parse $v$ are retained as candidate
transition.
The estimate DL is $L-l+l'$ where $L$ is the current whole DL, $l$ is
the current local description DL, and $l'$ is the new local
description DL under the candidate submodel~$m'$. The latter is zero
for expressions, which includes constant values, as there is nothing
left to encode.

The returned set of transitions may be incomplete for two reasons.
The first reason is that at most $K_p^2$ descriptions per example are
used to generate candidate submodels. Increasing~$K_p$ has an important
impact on the cost of computing example descriptions and transitions.
The second reason is that only a finite subset of expressions are
considered. In practice, the limitation is both on the size of
expressions (max. $S_e$), and on their number (max. $N_e$).

Finally, note that the returned DL is only an estimate about the
actual DL because it only re-parses the value of the
sub-descriptions. Re-parsing the whole value with the entire new model
may find a more compressive descriptions.

\subsection{Greedy Search: Finding Most Compressive Models}
\label{algo:search}

\begin{algorithm}[t]
  \caption{Greedy search of the most compressive task model}
  \label{algo:search:greedy}
  \begin{algorithmic}[1]
    \Require $T = (E,F)$: a task
    \Ensure $M$: the best found task model
    \Function{GreedySearch}{$T$}
    \State{$M \gets M_0$}
    \While{$\exists (v^i,v^o) \in E: v^o \not\in \textrm{\sc top}_{K}(\textrm{\sc predict}(M,v^i))$}
    \State{$\{D_{ij}\} \gets \{ \textrm{\sc describe}(M,v^i, v^o) \mid (v^i,v^o) \in E\}$}
    \State{$\Delta \gets \textrm{\sc top}_{K_t}(\textrm{\sc transitions}(M, \{D_{ij}\}))$}
    \State{$M'_{best} \gets \argmin_{M' \in \Delta} L(M',E)$}
    \If{$L(M'_{best},E) < L(M,E)$}
    \State{$M \gets M'_{best}$}
    \Else
    \State{{\bf break}}
    \EndIf    
    \EndWhile
    \State{{\bf return} $M$}
    \EndFunction
  \end{algorithmic}
\end{algorithm}

Algorithm~\ref{algo:search:greedy} defines function~{\sc GreedySearch}
that performs a greedy search over the space of task models, for a
given task~$T$. It starts from the initial task model (line~2) and,
while the current model~$M$ does not predict correctly all training
examples (line~3), it computes their descriptions (line~4) and from
there, a set of candidate transitions to refined models~$M'$
(line~5). The most compressive one is identified (line~6), and if it
is more compressive than the current model (line~7), it becomes the
new current model (line~8) and the process starts again, until it is
not possible to compress more (line~10).

Two parameters are involved in this search. $K$ is the number of
allowed attempts for prediction, e.g. 2 in ArcPrize. $K_t$ is the
maximum number of candidate transitions to consider, typically 100. It
is useful to control efficiency because computing the description
length $L(M',E)$ implies the costly computation of descriptions for
each example.

\subsection{Pragmatic Aspects}

There are a number of pragmactic aspects that were neglected in the
above formalization for the sake of simplicity but that play a
significant role in the implementation and experimental results. We
here describe them shortly.

\paragraph{Description lengths}
The two-part MDL definition $L(M,D) = L(M) + L(D \mid M)$ assumes that
$D$ is all the data to be modeled. However, in a program synthesis
setting like ARC, $D$ is only a small set of input-output pairs among
a large set of pairs that are valid for the task. We therefore
introduce a {\em rehearsal factor}~$\alpha$ in the above definition --
$L(M,D) = L(M) + \alpha L(D \mid M)$ -- in order to give more weight
to the data, and hence allow for more complex models. If the value of
$\alpha$ is too low, then the search may stop too early, missing key
decompositions in order to solve the task. If its value is too high,
the search may favor overly complex models w.r.t. examples, with a
risk of overfitting. The typical value used for ARC is $\alpha = 100$,
we compare with other values in the evaluation section.

Another difficulty that occurs in ARC is that, for some tasks, the
output grids are much smaller than the input grids. The consequence is
that the search concentrates entirely on compressing the inputs, not
paying attention to compressing the outputs. However, the latter is
necessary to find a predictive model whereas in may cases it is not
necessary to maximally compress the inputs. We therefore use a {\em
  normalized DL} that gives equal weight to the input and the output,
based on the initial model~$M_0 = (m^i_0,m^o_0)$,
\[ \hat{L}(M,E) = \frac{L(m^i,E^i)}{L(m^i_0,E^i)} +
  \frac{L(m^o,E^o)}{L(m^o_0,E^o)}, \] where $E^i$ and $E^o$
respectively denote the training inputs and training outputs. As
greedy search only proceeds with more and more compressive models, the
normalized DL is in the interval~$[0,2]$.

\paragraph{Transitions}
Most patterns combine two features: the verification that a value
matches the pattern, and the decomposition of the value into
parts. For instance, pattern {\bf Monocolor} verifies that a grid uses
a single color, and then decomposes it into a color and a mask. Some
patterns, like {\bf IsFull}, only feature verification; and other
patterns, like {\bf Vec}, only feature decomposition (a 2D vector can
always be decomposed into two integers). The difficulty is that pure
decomposition patterns are generally not compressive because they only
expose the internal structure of a value, like {\bf Vec} exposing the
two components of a 2D vector. As a consequence they are not selected
during search although they can be useful to expose a part that can be
compressed by a pattern or expression. We therefore extend the minimal
transitions~$m'$ by wrapping them by pure decomposition patterns~$P$,
in the form $P(\ldots, m', \ldots)$. An example is
${\bf Vec}({\it add}(x,1), ?)$ defining the first component of a
vector by an expression, the second component remaining to be
determined. This wrapping process may be repeated up to~$S_d$ times by
nesting several decompositions.

\paragraph{Search}
Every model considered by the search is guaranteed to parse all
training inputs, thanks to consistency
(Theorem~\ref{theo:consistency}). However the returned model may fail
to parse test inputs. This can be avoided by refining the search
algorithm so that models that fail to parse test inputs are
pruned. This is legitimate in ARC competitions where test inputs are
available to the learning system, and this similar to a human passing
an IQ test where both demonstration examples and test inputs are
available together. However, at least for most ARC tasks, it is not
necessary to take test inputs into account in order to come up with a
model that generalizes correctly to the unseen test examples.

The found models may overfit the training examples by including
regularities across them that are not essential to the task.  For
example, in task {\tt 47c1f68c} (Figure~\ref{fig:example}), all train
inputs and outputs have a black background but it could be any color
provided that it is the same color in both inputs and outputs. Adding
the parse checks about test inputs (see previous paragraph) is a
solution to ensure generalization to test examples but the model may
still overfit, failing to generalize to unseen examples. In the
example task, the test example also has black as background color. We
therefore add a {\em pruning phase} after the search phase. It
consists in replacing patterns and values in the input model by
unknowns whenever this does not reduce the prediction accuracy of the
model. In the example task, the constant value ${\it black}$ is
replaced by an unknown. By making the input model more general, we
enlarge the domain of application of the task model, hence
generalization to unseen examples. Pruning does not apply to the
output model because it would increase the number of predicted
outputs, hence reducing model accuracy.

\begin{table}[t]
  \centering
  \begin{tabular}{llr}
    parameter & description & default \\
    \hline
    $N_p$ & max nb. parsed descriptions & 100 \\ 
    $K_p$ & nb. top parsed descriptions & 3 \\ 
    $N_g$ & max nb. generated descriptions & 3 \\ 
    $K_g$ & nb. top generated descriptions & 3 \\ 
    $S_e$ & max expression size & 6 \\ 
    $N_e$ & max nb. expression candidates & 1000 \\ 
    $K_t$ & max nb. evaluated transitions & 100 \\ 
    $K$ & nb. allowed prediction attempts & 3 \\ 
    $\alpha$ & rehearsal factor & 100 \\
    $S_d$ & max nb. decompositions wrapping transitions & 1 \\ 
    ${\it testcheck}$ & parse-check of test inputs & true \\
    ${\it pruning}$ & pruning phase & true \\
  \end{tabular}
  \caption{All parameters involved in MADIL algorithms, along with their default value}
  \label{tab:params}
\end{table}

\paragraph{Parameters}
Table~\ref{tab:params} summarizes all parameters used in the above
algorithms, along with a short description and their default value in
experiments.

\section{Advanced Contributions}
\label{sec:advanced}

This section presents advanced contributions that significantly
improve performance on ARC: collections of parts, patterns that
depends on values, and improved search based on MCTS.

\subsection{Collections of Parts}
\label{collections}

\begin{table}[t]
\centering
\begin{tabular}{|l|l|}
  \hline
  type & patterns \\
  \hline
  \hline
  {\it Sprite} & {\bf Objects}(size$:$ {\it Vec},\ seg$:$ {\it Seg},\ order$:$ {\it Order},\ n$:$ {\it Int}, \\
       & \hspace{1cm} objects$:$ [{\it Object}]), \\
       & {\bf Metagrid}(sepcolor$:$ {\it Color},\ borders$:$ {\it Mask},\ dims$:$ {\it Vec}, \\
       & \hspace{1cm} heights$:$ [{\it Int}],\ widths$:$ [{\it Int}],\ subgrids$:$ [[{\it Sprite}]]), \\
       & {\bf ColorRow}(size$:$ {\it Int},\ colors$:$ [{\it Color}]), \\
       & {\bf ColorCol}(size$:$ {\it Int},\ colors$:$ [{\it Color}]), \\
       & {\bf ColorMat}(size$:$ {\it Vec},\ colors$:$ [[{\it Color}]]) \\
  \hline
  {\it Object} & {\bf Obj}(pos$:$ {\it Vec},\ sprite$:$ {\it Sprite}) \\
  \hline
  [[{\it Color}]] & {\bf AsGrid}(grid$:$ {\it Grid}) \\
  \hline
  {\it ColorMap} & {\bf DomainMap$_C$}(colors$:$ [{\it Color}]) \\
  \hline
  [{\it Int}] & {\bf Range}(start$:$ {\it Int},\ step$:$ {\it Int}) \\
  \hline
  $X^n$ & {\bf Cons$_{d<n}$}(head$:$ $X^{n-1}$,\ tail$:$ $X^n$), {\bf Repeat$_{d<n}$}(item$:$ $X^{n-1}$) \\
  \hline
\end{tabular}
\caption{Patterns with sequences of parts by type}
\label{tab:arc:patterns:seq}
\end{table}

The patterns defined so far have a fixed number of parts
(Section~\ref{pattern}). However, we can think of several patterns
that result in variable numbers of parts: segmenting a sprite into a
collection of objects; splitting a grid into subgrids; decomposing a
1D grid into a sequence of colors. Rather than allowing a variable
number of arguments in patterns, which would be confusing when there
are different types of parts, we extend the domain of values with {\em
  sequences} of values, sequences of sequences of values, and so
on. In the future other kinds of collections could be introduced.
Given a type of values~$X$, [$X$] denotes a sequence of $X$-values,
and [[$X$]] a sequence of sequences of $X$-values. More generally,
$X^n$ denotes $n$ layers of sequences around $X$-values, which we call
a {\em $n$D sequence}. Hence [$X$] is a 1D sequence and [[$X$]] is a
2D sequence (aka. matrix when regular). This is analogous to tensors,
which are $n$D arrays, without the regularity constraint,
i.e. subsequences need not have the same length.

\begin{figure}[t]
  \centering
  \includegraphics[width=0.8\textwidth]{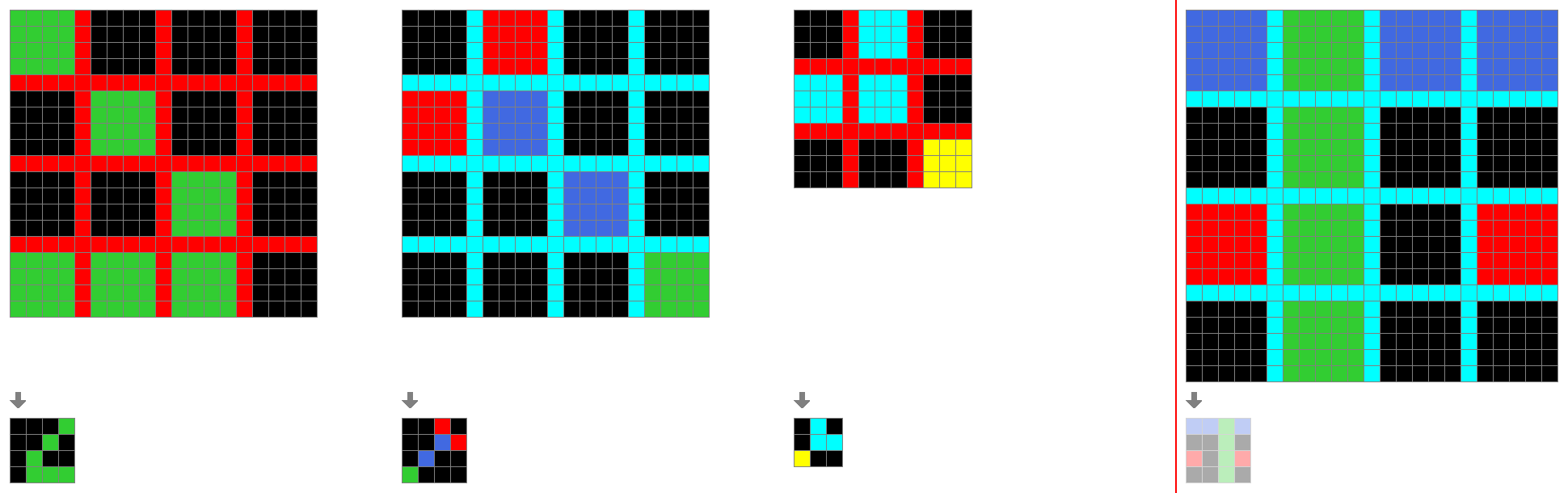}
  \caption{Task 9f236235 (inputs at the top, outputs at the bottom, test on the right)}
  \label{fig:example:seq}
\end{figure}

\begin{arcdef}
  Table~\ref{tab:arc:patterns:seq} lists ARC patterns that involve
  sequences. They add to the patterns listed in
  Table~\ref{tab:arc:patterns}.
  {\bf Objects} decomposes a sprite into a sequence of objects
  according to some segmentation mode (type {\it Seg},
  e.g. 4-connectedness and same color) and to some ordering criteria
  (type {\it Order}, e.g. by decreasing area). There is also a part
  for the size of the sprite as it cannot be deduced from the
  objects. The length of the sequence, here the number of objects~$n$,
  is represented apart from the sequence contents in order to
  facilitate reasoning on the two independently. ${\bf Obj}$ is a pure
  decomposition pattern for objects, separating their position and
  sprite.
  
  {\bf Metagrid} decomposes a sprite (or a grid or a mask) into a
  matrix of subgrids, according to some separating
  frontiers. Figure~\ref{fig:example:seq} shows a task where inputs
  follow this pattern. All subgrids here have the same size but this
  need not be the case. The other parts are the separator color (red
  in the first example), the presence or not of borders represented as
  a 2x2 mask (4 sides, here no borders), the dimensions~$(k,l)$ of the
  matrix of subgrids (4x4 in the first example), the heights and
  widths of subgrids (all subgrids in a metarow have the same weight,
  and all subgrids in a metacolumn have the same width).
  
  {\bf ColorRow}, {\bf ColorCol} and {\bf ColorMat} decompose sprites
  (or grids or masks) into 1D or 2D sequences of colors. The former
  two are only valid for 1D grids. The other way around, {\bf AsGrid}
  converts a 2D sequence of colors, when regular, into a grid. Such
  type conversions are defined because different types provide
  different patterns.
  
  {\bf DomainMap$_C$} reduces a color map to its range colors given a
  fixed set of domain colors~$C$.
  {\bf Range} decomposes an integer sequence that follows an
  arithmetic progression into a start value and a step value. Note
  that there is no part for the stop value or sequence length because
  sequence lengths are encoded by the patterns that introduce the
  sequences.
  
  {\bf Cons$_d$} and {\bf Repeat$_d$} are polymorphic patterns that
  only act on $n$D sequences, independently of the type of their
  contents. They are parametrized by a depth~$d \in [0,n[$ to specify
  the concerned axis. {\bf Cons$_d$} distinguishes the head from the
  tail of every subsequence at depth~$d$, when not empty. For
  instance, {\bf Cons$_1$} decomposes the 2D sequence
  $[[0,1,2],[3,4,5]]$ into the head $[0,3]$ and the tail
  $[[1,2],[4,5]]$, while {\bf Cons$_0$} decomposes it into the head
  $[0,1,2]$ and the tail $[[3,4,5]]$.
  {\bf Repeat$_d$} verifies that all subsequences at depth~$d$ are
  repeats of the same value, which can be compressed to that
  value. For instance, {\bf Repeat$_1$} compresses the 2D sequence
  $[[1,1],[2,2,2]]$ into the 1D sequence $[1,2]$.
\end{arcdef}

There is a difficulty arising from patterns that introduce
sequences. Looking at the task in Figure~\ref{fig:example:seq}, each
input grid matches pattern {\bf Metagrid}, and is therefore decomposed
into a matrix of subgrids. Each subgrid matches pattern {\bf
  Monocolor} but this pattern expects grids, not matrices of grids. We
can resolve this type mismatch by allowing the item-wise application
of patterns on $n$D sequences of the expected type. This means that
the general definition of the pattern is: ${\it Grid}^n$ = {\bf
  Monocolor}(color$:$ ${\it Color}^n$, mask$:$ ${\it Mask}^n$). This
generalization applies to all patterns, including those that introduce
or consume layers of sequences. It also applies to functions.

The application of {\bf Monocolor} on the matrix of subgrids
introduces a matrix of colors and a matrix of masks as parts. From
there, the matrix of colors can be converted into a grid through
pattern {\bf AsGrid}, and the output grid can finally be identified as
the mirror of this grid.

\begin{arcdef}
  In parallel to patterns, new functions are introduced to handle
  sequences of any type: length of a sequence; tail, reversal, and
  rotation of a 1D sequence; transposition and flattening of 2D
  sequences; aggregations (sum, min, max, argmin, argmax, most common,
  least common); and lookup of items and substructures at a few key
  indexes (e.g., first: $x[0]$, last: $x[-1]$, top-left: $x[0,0]$).

  There are also type-specific functions involving sequences: bitwise
  logical operators on sequences of masks; colors of a grid (in
  descending frequency); halves and quadrants of a grid; conversion of
  a sequence or matrix of colors into a grid; matrix of relative
  positions between a sequence of objects (for translation {\em at}
  and translation {\em onto}).
\end{arcdef}

The description length of a $n$D sequence is simply the sum of the
description lengths of its items, under the assumption that items are
independent.  It is the job of sequence patterns to exhibit
dependencies, and use them to better compress the sequence (e.g.,
pattern {\bf Range}). There is no need to encode sequence lengths
because they are all specified as parts of the patterns that introduce
those sequences.

\subsection{Dependent Patterns}
\label{dependent:patterns}

There are common ``patterns'' that are appealing but fit neither
patterns nor expressions in a satisfactory way. For example, in a
number of ARC tasks, the output grid (or a part of it) can be obtained
by cropping the input grid (or a part of it). On one hand, this
suggests a function $subgrid = {\it crop}(grid,position,size)$ but
this requires to generate expressions for many combinations of grids,
positions, and sizes. For tractability reasons, only a few constant
positions and sizes would be considered, possibly missing the correct
values.
On the other hand, a pattern
$subgrid = \textrm{\bf Crop}(grid,position,size)$ could be defined but
the decomposition faces an even higher challenge by having to generate
larger grids containing the subgrid.

Taking the problem upside down, let us look at the tractable
computations. There are two of them.
\begin{align*}
  grid, position, size \leadsto subgrid \\
  grid, subgrid \leadsto position, size
\end{align*}
The first computation is the cropping function, and it is
deterministic. The second computation is non-deterministic because the
subgrid may have several occurrences in the grid. Moreover, the second
computation can be seen as the reverse of the first computation, for
any fixed grid. This suggests a pattern parametrized by the grid,
where the subgrid is the whole and the position and size are the
parts.
\[ subgrid = \textrm{\bf Crop}[grid](pos: {\it Vec},\ size: {\it Vec}) \]
We talk about {\em dependent patterns} by analogy with dependent
types, which are types whose definition depends on a value.
The grid parameter must be known for both composition and
decomposition but it can be the result of a computation. Therefore,
any expression defined on the environment -- including constant values
-- is acceptable. The advantage compared to a function is that the
position and the size can be computed efficiently at parse time, just
by looking for occurrences of the subgrid in the grid.
In practice, {\bf Crop} is introduced in the output model only, and
its grid parameter is any grid variable from the input model. More
complex expressions are not considered as parameters for tractability
reasons.

\begin{table}[t]
\centering
\begin{tabular}{|l|l|}
  \hline
  type & patterns \\
  \hline
  \hline
  {\it Sprite} & {\bf Crop}[{\it Sprite}](pos$:$ {\it Vec},\ size$:$ {\it Size}), \\
       & {\bf Recoloring}[{\it Sprite}](cmap$:$ {\it ColorMap}) \\
  \hline
  $X^n$ & {\bf Index}[$X^{n+k}$](idx$:$ [{\it Int}]) \\
  \hline
\end{tabular}
\caption{Dependent patterns by type}
\label{tab:arc:patterns:dep}
\end{table}

\begin{arcdef}
  Table~\ref{tab:arc:patterns:dep} lists a few dependent patterns for
  ARC. In addition to {\bf Crop} that is described above, we have {\bf
    Recoloring} that compresses a sprite into the recoloring of the
  sprite parameter, when such a recoloring exists of course. {\bf
    Index} locates the whole value, an $n$D sequence, as a
  substructure of the parameter, an $(n+k)$D sequence, and compresses
  it to its index, a 1D sequence of length~$k$. It enables to reason
  on this index location, unlike with the indexing function that uses
  only a few constant indexes.
\end{arcdef}

\subsection{Monte Carlo Tree Search} 
\label{mcts}

As the number of patterns and functions grows there is a higher risk
that greedy search falls into a local minimum. It is desirable to
allow for more exploration, in addition to the exploitation of the MDL
principle, in order to recover from wrong steps on the search path. A
well-established method to balance exploitation and exploration is
Monte Carlo Tree Search (MCTS)~\cite{MCTS2012}. The states are task
models~$M$, the root state is the empty task model~$M_0$, and the
actions are transitions. The value of a state is defined as
$1 - \frac{1}{2}\hat{L}(M,E) \in [0,1]$, giving higher value to the
more compressive models, and value~0 to the root state. We define the
four core steps of MCTS as follows:
\begin{itemize}
\item {\em Selection}: the node to expand is chosen according to the
  UCB1 policy with an exploration constant $c = \sqrt{2}$. when two
  children have the same score, the one with lower DL is prefered,
  i.e. the more promising children are explored first.
\item {\em Expansion}: the selected node is expanded by the top $K_c$
  transitions ($c$ for children), based on the normalized DL of the
  $K_t$ evaluated transitions. $K_c$ is therefore the maximum
  branching factor of the search tree.
\item {\em Rollout}: greedy search (Algorithm~\ref{algo:search}) is
  used for rollout, only one is done as it is deterministic. There
  is no limit in depth, the rollout stops when the description
  length cannot be decreased further.
\item {\em Backpropagation}: the backpropagated value is based on the
  description length of the final task model in the rollout.
\end{itemize}
The search stops when a correct task model is found, or when a time
budget is consumed.

\section{Evaluation}
\label{eval}

We evaluate the MADIL approach on ARC in terms of performance on the
public and private sets of ARC tasks, of search efficiency, and of
sensitivity to the different parameters. We also analyse the failures
and limits of the approach.
Our experiments were run with a single-thread
implementation\footnote{Open source available at
  \url{https://github.com/sebferre/ARC-MDL}} on Fedora~32, Intel Core
i7x12 with 16GB memory.
The learning and prediction logs are available in the GitHub
repository.

\subsection{Performance and Solved Tasks}

\begin{table}[t]
  \centering
  \begin{tabular}{lrrr}
    method & training & evaluation & private \\
    \hline
    \multicolumn{4}{c}{\em DSL + search} \\
    \hline
    MADIL v3.5 & 32.1\% & 15.1\%  & 7.0\% \\ 
    Icecuber (Kaggle'20) & 44.7\% & 30.8\% & 20.6\% \\ 
    M. Hodel (ARCathon'22) & & & 6.0\% \\
    \hline
    \multicolumn{4}{c}{\em language models} \\
    \hline
    MindsAI (ArcPrize'24) & & & 55.5\% \\
    The Architects (ArcPrize'24) & & & 53.5\% \\
    o3 & & 82.8\% & n/a \\
    J. Berman & & 58.5\% & n/a \\
    R. Greenblatt & & 42.0\% & n/a \\
    o1-preview & & 21.0\% & n/a \\
    GPT-4o & & 9.0\% & n/a \\
  \end{tabular}
  \caption{Performance of MADIL and a few methods on the three ARC datasets}
  \label{tab:performance}
\end{table}

Table~\ref{tab:performance} reports the performance of our MADIL
approach on three datasets (first row): the 400 public training tasks,
the 400 public evaluation tasks, and the 100 private tasks used in the
different competitions.  We have only looked at the training tasks for
the design of MADIL, we have never looked at the evaluation tasks,
which therefore constitute a robust evaluation. The performance
measure is micro-accuracy, counting fractional task scores for each
successful test instance. For instance, when one out of two test
instances are correctly predicted, the task score is 0.5.
The reported version of MADIL is v3.5, the last submission at
ArcPrize'24. MADIL's performance is compared to the winning approaches
of successive competitions, as well as to various LLM-based
approaches.\footnote{All figures are from the ArcPrize website at
  \url{https://arcprize.org/}, except for Icecuber on training and
  evaluation datasets for which we used the open source code.}  The
approaches that rely on proprietary LLMs (from o3 to GPT-4o) are not
legible for the private dataset because they require internet access.

On the 400 public evaluation tasks, MADIL finds a solution for 69
tasks, i.e. a task model that is correct on all training examples. Out
of them, 57 are correct on all test examples, hence a generalization
rate at~82\%. There are two more tasks with correct predictions, where
the found model is correct on one training example out of two. This
shows that MADIL can sometimes be robust against a training example
that is not understood.

For recall, two attempts were allowed at ArcPrize'24. However, most
predictions are correct on first attempt. Indeed, for 53/57 solved
tasks, all predictions were correct on first attempt.

\begin{figure}[t]
  \centering
  \begin{minipage}{0.45\textwidth}
    \centering
    \includegraphics[width=\linewidth]{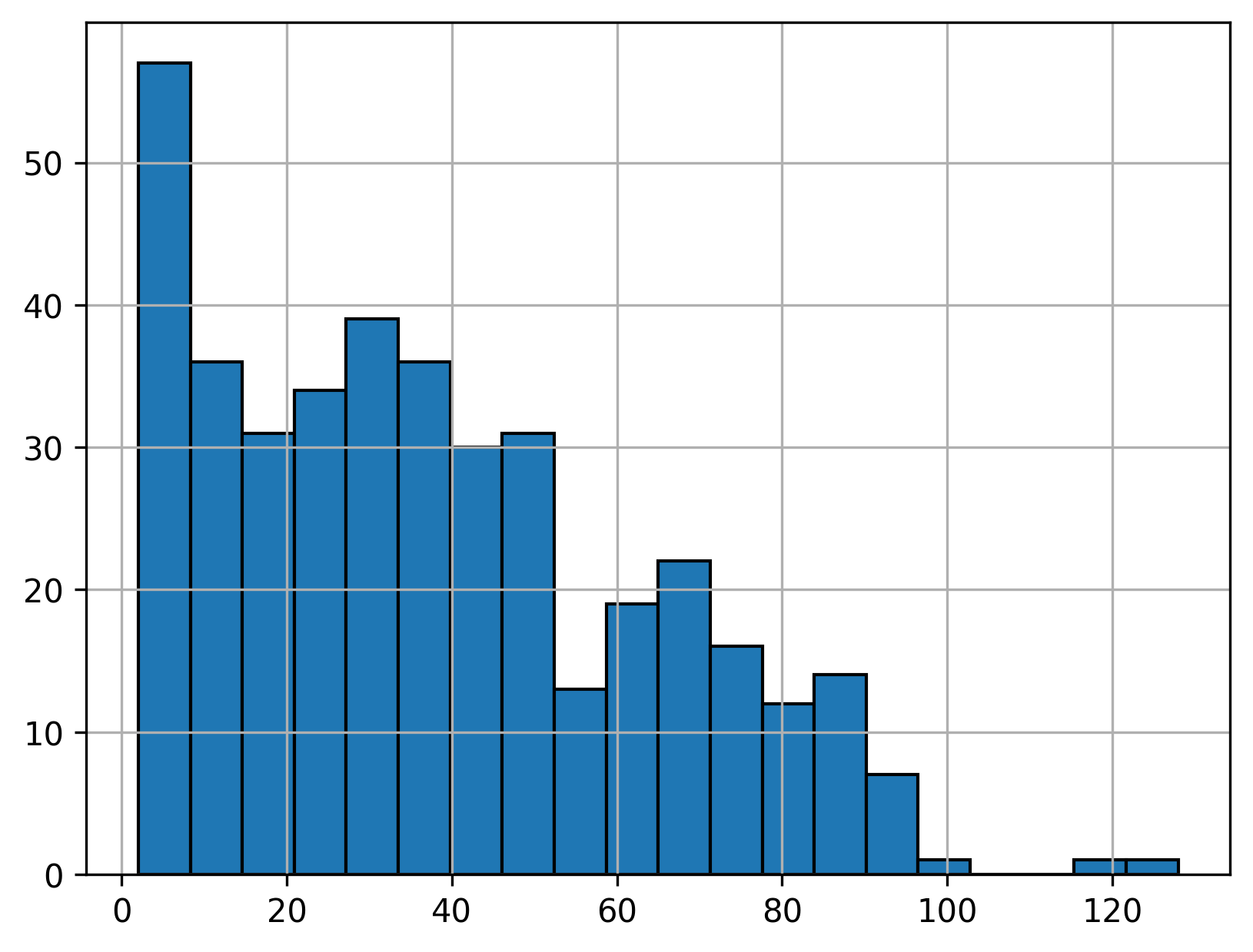}
  \end{minipage}
  \hfill
  \begin{minipage}{0.45\textwidth}
    \centering
    \includegraphics[width=\linewidth]{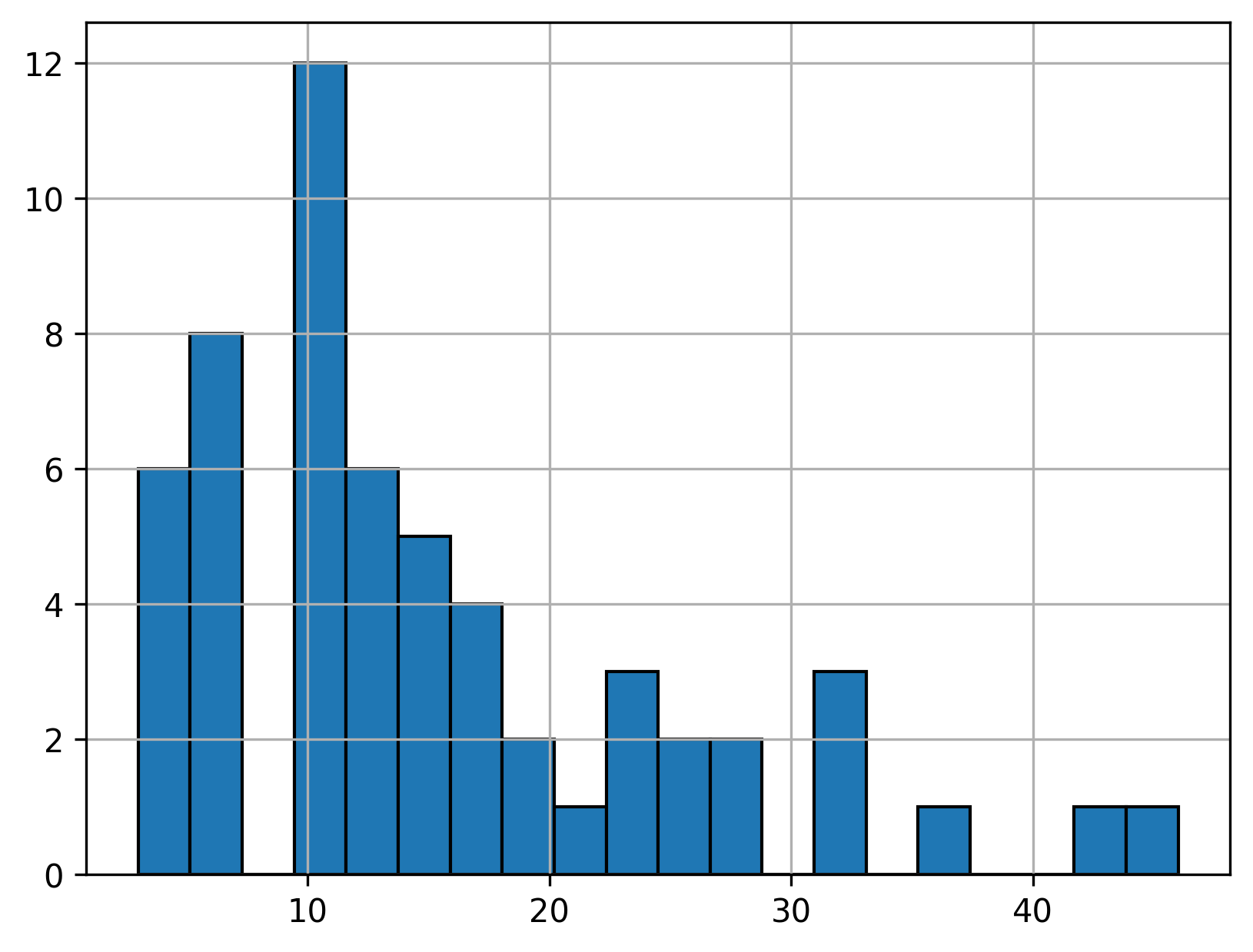}
  \end{minipage}
  \caption{Distribution of model sizes: all evaluation tasks (left), and correct solutions only (right).}
  \label{fig:modelsizes}
\end{figure}

Figure~\ref{fig:modelsizes} shows the distributions of the size of
models found, over all evaluation tasks on the left, and over the 57
correct solutions only on the right. The size of a model is the number
of values, variable references, functions, patterns, and unknowns it
is composed of. The model size for correct solutions ranges from 3 to
46, and is 15.2 on average. This means that a brute-force search on
MADIL's DSL would have to reach depth up to 46 in order to find such
models. This is intractable, existing brute-force search approaches
are limited to depth 3-4. Only 4 tasks have a correct solution with
size less or equal to~4. This demonstrates the relevance of
description lengths to guide the search.
More in detail, Table~\ref{tab:model:usage} shows the usage of the
different patterns and functions in the 400 found models. The most
common grid patterns are {\bf BgColor}, {\bf Objects}, and {\bf
  Monocolor} that together enable to decompose grids into collections
of monocolor objects over some background color, and also {\bf Motif}
that recognizes various geometric regularities in a grid. There are
also the sequence-related patterns {\bf Cons}, and {\bf Repeat}.
Common functions are substraction ({\tt -}) and addition ({\tt +}) to
decrement or increment quantities by small constants, access to the
two components $i$ and $j$ of vectors, the area and pairs of halves of
grids, accessing (e.g., {\it index}, {\it tail}) or aggregating (e.g.,
{\it max}) a sequence. We also see that color constants (COLOR) are
very common.

\begin{table}[t]
  \centering
    {\small
  \begin{tabular}{lr@{\hspace{4em}}lr@{\hspace{4em}}lr}
{\bf Cons} & 1075 & {\bf Metagrid} & 58 & argmax & 13 \\
{\bf Objects} & 659 & {\bf Crop} & 57 & applySymGrid & 13 \\
{\tt -} & 651 & min & 53 & {\bf ColorSeq} & 13 \\
{\bf BgColor} & 624 & colorCount & 47 & top & 11 \\
COLOR & 482 & majorityColor & 46 & reverse & 10 \\
{\bf Monocolor} & 435 & colors & 44 & {\bf Recoloring} & 10 \\
index & 324 & {\tt *} & 43 & not & 8 \\
{\bf Motif} & 281 & {\bf Square} & 42 & rotate & 7 \\
{\bf Empty} & 260 & left & 41 & gridOfColorSeq & 7 \\
{\bf Repeat} & 230 & translatedOnto & 34 & closeSym & 7 \\
cast & 159 & right & 33 & or & 5 \\
i & 149 & mostCommon & 32 & middle & 4 \\
{\tt +} & 128 & maskOfGrid & 32 & flattenByCols & 4 \\
area & 113 & {\bf MakeGrid} & 32 & compose & 4 \\
max & 104 & sum & 31 & ijTranspose & 3 \\
{\tt /} & 100 & {\bf ColorMat} & 31 & xor & 2 \\
j & 99 & relativePos & 30 & unrepeat & 2 \\
halvesH & 85 & direction & 30 & middleCenter & 2 \\
tail & 78 & {\bf Line} & 30 & center & 2 \\
norm & 75 & bottom & 28 & gridOfColorMat & 1 \\
{\bf Full} & 74 & {\bf Point} & 26 & {\bf Swap} & 1 \\
{\bf Index} & 71 & argmin & 18 & {\bf Range} & 1 \\
    halvesV & 62 & transpose & 16 &  \\
  \end{tabular}
  }
  \caption{Usage of patterns (bold), color literals (COLOR), and
    functions in evaluation models.}
  \label{tab:model:usage}
\end{table}

We can compare the correct predictions of MADIL vs Icecuber's approach
because we could rerun it. We only consider predictions as correct
when all test output grids are correctly predicted. In this sense,
MADIL is correct on 59 evaluation tasks whereas Icecuber is correct on
117 evaluation tasks, hence twice more. Beyond those counts, we wanted
to know to which extent the two sets of correct tasks overlap. It
happens that this overlap is important, 43 tasks, but there are still
16 tasks that are solved solely by MADIL, hence showing some
complementarity between the two approaches.

\subsection{Search Efficiency}

\begin{table}[t]
  \centering
  \begin{tabular}{l|r|r|rr}
    & all models & final models & \multicolumn{2}{|c}{best found model} \\
    task set & & & search depth & rank \\
    \hline
    all (400) & 177 & 28.7 & 17 & 1.6  \\
    solutions (69) & 45 & 3.3 & 11 & 1.3 \\
  \end{tabular}
  \caption{Search efficiency measures: number of visited models, number
    of final models (ends of MCTS rollouts), and search depth and
    rank of the best found model. Averaged over two sets of tasks.}
  \label{tab:search:effort}
\end{table}

Search efficiency is the key to scale in expressive power. Indeed, the
search space grows exponentially with the number of DSL
primitives. According to Chollet~\cite{Chollet2019}, intelligence
amounts to the ``efficient acquisition of new skills.''
Table~\ref{tab:search:effort} measures the search effort by counting
the number of models considered during the search. We distinguish
``all models'' that includes intermediate models on the paths from the
initial model to the ``final models'', i.e. the end of MCTS
rollouts. Those ``final models'' represent candidate models for the
task. We also measure for the best found model, its search depth and
its rank among final models.  We average those measures on two
different task sets: the 400 evaluation tasks (first row), and the 69
tasks for which an actual solution was found.
The first row shows that although more than 28 candidate models are
considered over all tasks, the best model is found quickly with an
average rank~1.6. The total number of visited models is 177 on
average, which seems low for a 180s budget. This is explained by the
cost of computing for each model the candidate transitions, and
computing for each refined model its best descriptions and description
lengths in order to select the top-k transitions. The average search
depth is 17, a high value for a search space that grows exponentially
with depth. The maximum search depth is even 43, and a solution is
found at that depth (although it does not generalize to test
instances).
The second row shows a lower search effort, 45 vs 177, because
solutions are found very early at average rank~1.3. Actually, 54/69
solutions are found on the first MCTS rollout, i.e. with greedy
search. This demonstrates the accuracy of description lengths and the
power of the MDL principle. The average search depth is lower at~11
because in the absence of a solution, the system builds more complex
models in an effort to find a solution.

\begin{figure}[t]
  \centering
  \includegraphics[width=0.8\textwidth]{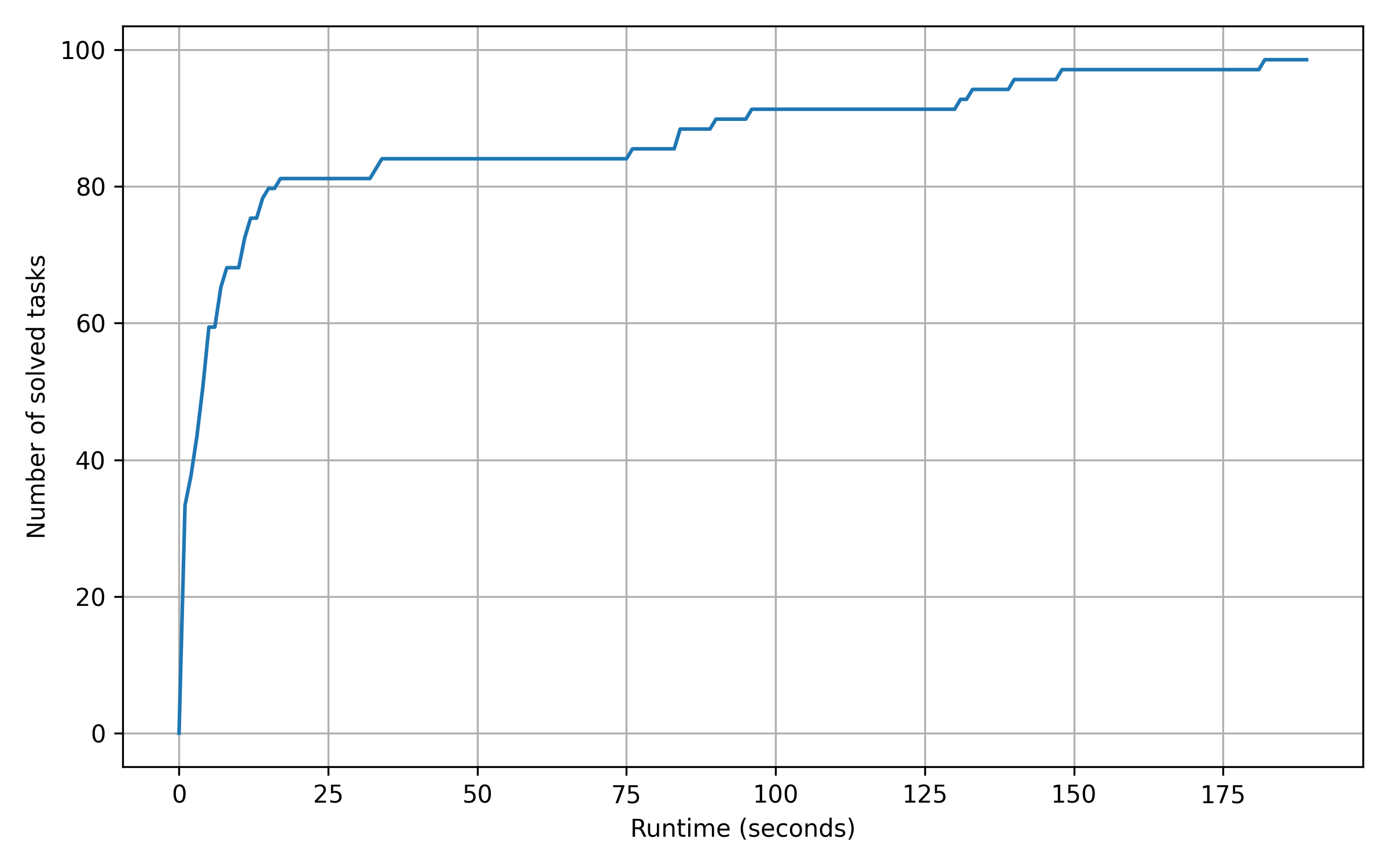}
  \caption{Percentage of solved tasks under a given runtime.}
  \label{fig:search:delay}
\end{figure}

Figure~\ref{fig:search:delay} shows how quickly the 69 task solutions
are found by measuring the percentage of tasks solved, out of 69, as a
function of max runtime. Strikingly, 80\% solutions are found in the
first 20s, hence in about 10\% of the time budget.
By profiling the share of runtime among different functions, we
observe that half of the computation time is spent on computing the
example descriptions for some model (function {\sc describe}, see
Algorithm~\ref{algo:describe}), and the other half is spent on
computing the transitions (function {\sc transitions}, see
Algorithm~\ref{algo:transitions}) for some model. Two third of the
latter is related to expression transitions.

\subsection{Parameter and Ablation Study}


In this section, we perform a parameter study. The objective is to see
how variations in the parameters listed in Table~\ref{tab:params}
affect performance. More precisely, we focus on the 71 evaluation
tasks that are satisfactorily solved\footnote{Correct on at least 2/3
  training examples, and allowing 2 attempts for predictions.} in any
MADIL version from v3.2 to v3.6, and we explore the effect of
decreasing each parameter, i.e. lowering the search effort.  We
distinguish between ``no effect'', ``small effect'' when only a few
tasks are lost, and ``significant effect'' when a significant
proportion of tasks are lost.
We observe significant effect when:
\begin{itemize}
\item Sampling a single grid description ($N_p=1$), or keeping only
  the top example description ($K_p=1$), loosing respectively 19 and 6
  tasks. Sampling multiple descriptions is important to find good
  descriptions, and multiple example descriptions is important to find
  good transitions. However, having $N_p=10$ (vs 100) or $K_p=2$ (vs
  3) only has a small effect (1 task lost).

\item Limiting expression size $S_e=1$, loosing 27 tasks. Complex
  expressions hence appear important. However, there is no task loss
  in our task sample for $S_e=5$ (vs 9), and lowering this
  parameter saves time and allows for more search in given budget.

\item Using only the most promising transition ($K_t=1$) or using
  greedy search, loosing respectively 3 and 4 tasks. This shows that,
  despite greedy search works in most cases, search remains useful to
  recover from local minima.

\item Not using the rehearsal factor for computing description lengths
  ($\alpha=1$), loosing 6 tasks. However, there is no effect with
  $\alpha=10$ (vs 100).

\item Using neither input test check nor pruning, loosing 6 tasks.
\end{itemize}
Overall, the whole method appears very robust because strong effects
only occur for $N_p=1$ and $S_e=1$, and are moderate to null in all
other cases. The only positive effect was found for $N_e \in [30,100]$
(vs 1000), earning 1 task. This is explained by the cost of
enumerating candidate expressions that here gets converted in more
search.

\subsection{Failed Tasks and Limits}

We study the limits of our approach by analyzing failed tasks. There
three kinds of failures: (a) solution was found but it does not
generalize, (b) a solution seems to exist but it was not found, (c)
there is no solution in the current implementation.

\paragraph{Generalization failures}
We here look at the 21 training tasks for which a solution was found
but it is incorrect on some or all test instances. The major cause is
spurious expressions that manage to compensate for the wrong
decompositions of grids. The risk to find spurious expressions rises
when there are fewer training examples, and when allowing for larger
expressions (parameter~$S_e$). On inspection, they look really
contrived and artificial so that, according to the MDL principle, they
should be strongly penalized. This suggests revising the description
lengths of expressions, and also the balance between the model and the
data (parameter~$\alpha$).

\paragraph{Search failures}
We here look at 62 training tasks for which no solution was found but
that seem solvable in the current implementation.
The most common cause (21 tasks) is that non-compressive transitions
are sometimes necessary in order to reach a solution, which does
compress more in the end. This suggests two improvements. First, one
could allow for non-compressive transitions, at least in a controlled
way like the {\em late acceptance hill-climbing
  heuristics}~\cite{Burke2017}. Second, it may be that some
description lengths are poorly designed, overestimating the
information contents of some descriptions. It is in particular the
case with object segmentation where connectivity is not sufficiently
taken into account.

Another common cause (15 tasks) is that the top-$K_p$ descriptions are
not diverse enough so that key transitions are missed. Indeed, recall
that candidate transitions are derived from the observation of those
descriptions. There is a combinatorial problem with the computation of
descriptions. Considering a model that describes a grid as a set of
objects over some background color, without additional constraints,
there is a description for each background color, each segmentation
mode, each object ordering, etc. Description lengths help to choose
the more relevant descriptions but those choices may be inconsistent
from one example to another. This suggests the following
improvements. First, the value of $K_p$ could be increased but this
has a cost, and this is a weak solution in face of
combinatorics. Second, one could accept candidate transitions that are
derived from only a subset of the training examples, although this may
add a lot of extra-cost to the computation of transitions. Third, the
combinatorial aspect could be broken by modifying the parsing process
to compute sets of values for each variable of the model, instead of
sets of descriptions, which are mappings from variables to values.

Other causes for search failures (13 tasks) are poor rankings of
candidate transitions that put the solution far away, spurious
expressions that attracts search in the wrong subspace, or
computation-intensive steps that slow down search.

\paragraph{Missing primitives}
There are still many missing primitives to cover the core knowledge
priors assumed by ARC:
\begin{itemize}
\item non-centered motifs, comparability of rotated motif cores, and
  complex compositions of different primitive motifs;
\item shapes that are best understood as drawing algorithms, like
  spiral or staircase shapes;
\item more robust object segmentation, especially w.r.t. overlapping,
  disconnected objects, and exploiting object similarities and
  differences across examples, e.g. prefering identical shapes when
  possible;
\item topological relationships between objects, like ``adjacent to'',
  ``on top of'' or ``opposite'', beyond the current vector-based
  relative positions;
\item support for grid orientation invariance;
\item piecewise decomposition of an output (sub)grid as a stack of
  pieces, this is a key ingredient of Icecuber's approach;
\end{itemize}
The current DSL is also limited in its handling of sequences. It is
often necessary to partition a sequence into two subsequences
according to some pattern, but there is a chicken-and-egg problem to
find the partition and the pattern. Another limit is when an output
sequence is not just a mapping of an input sequence but some items
have to be added or removed.
Finally, we avoided almost completely n-ary expressions,
i.e. expressions that use several variables, in order to avoid a
combinatorial explosion in the enumeration for expressions. This is
partly compensated by relational matrices over sequences. Defining
more dependent patterns like $P[x](?y)$, in place of binary functions
like $f(x,y)$, could help in this issue.

\paragraph{Limits of the current approach}
There are also limits that do not seem easy or even possible to
compensate by adding primitives to the DSL.
A first limit is that patterns are parsed on one grid at a time, hence
missing the global picture. It would be beneficial, and sometimes
necessary, to parse all input grids or all output grids together. For
instance, this would help to choose a segmentation mode or motif in a
consistent way across examples. It would also be beneficial to have
patterns over input and output grids together rather than
separately. For instance, this would enable to identify what stays and
what changes, hence guiding the segmentation process.

Another limit is the absence of key constructs such as conditionals
and recursion. They are not the majority but some tasks rely on
them. There is earlier work on program synthesis that could serve as a
basis~\cite{FlashFill2013,Mulleners2023,Rule2024}.

A last limit we give here -- there are certainly others -- is the fact
that candidate expression transitions are retrieved by value in a DAG
of expressions, hence relying on a simple equality between expression
results and the expected value. It would be much more powerful to have
a kind of Content-Based Information Retrieval (CBIR) system. For
example, a grid query could return grids that contain it or that are a
symmetry of it; a sequence query could return sequences that include
it, or that are simple transformations of it (e.g., dedupe,
reverse).

\section{Conclusion and Perspectives}
\label{conclu}

We have introduced and described in detail MADIL, a novel approach to
program synthesis that is based on descriptive models and on the
Minimum Description Length (MDL) principle. In this setting, a program
maps an input to an output by decomposing the input with an input
model, and then by composing the output with an output model that is
fed the input decomposition. Learning a program from input-output
examples follows the MDL principle, searching for a pair of models
that best compress the examples.
Applying MADIL to some domain, i.e. some family of tasks, mostly
amounts to define a collection of patterns, where each pattern
supports the decomposition and composition of some values. In this
paper, we focus on the grid-to-grid tasks of the Abstraction and
Reasoning Corpus (ARC).
The main advantage of our approach compared to other approaches is the
efficiency of the MDL-based search. Compared to brute-force search, it
enables deep and narrow search rather than shallow and broad
search. This allows for more low-level primitives, and this scales
better with the number of primitives. Compared to LLM-based approach,
it is much more frugal, it does not rely on heavy generate-and-test,
and its predictions are interpretable.

Future work will focus on identifying and addressing the limits of
MADIL, using ARC as a stimulating benchmark. Adding more primitives
(patterns and functions) is unlikely to solve ARC, or may result in
ad-hoc solutions that do generalize well to other domains. Among the
identified limits, there are: the combinatorial complexity of
decompositions, increasing with model size; the fact that the input
and output models are learned separately, hence missing insights from
commonalities between inputs and outputs; the lack of conditional
models to distinguish differents cases in some tasks; the difficulty
to match input part sequences to output part sequences, especially
when there is no obvious ordering; taking into account constraints
between different parts.
At another level, an important research problem is the learning of the
MADIL primitives from a collection of training tasks, relying on a
general programming language of some sort to define the primitives.


\begin{thebibliography}{10}
\expandafter\ifx\csname url\endcsname\relax
  \def\url#1{\texttt{#1}}\fi
\expandafter\ifx\csname urlprefix\endcsname\relax\def\urlprefix{URL }\fi
\expandafter\ifx\csname href\endcsname\relax
  \def\href#1#2{#2} \def\path#1{#1}\fi

\bibitem{KriSutHin2012nips}
A.~Krizhevsky, I.~Sutskever, G.~E. Hinton, Imagenet classification with deep
  convolutional neural networks, Advances in neural information processing
  systems 25 (2012) 1097--1105.

\bibitem{Silver2016alphago}
D.~Silver, A.~Huang, C.~J. Maddison, et~al., Mastering the game of {Go} with
  deep neural networks and tree search, Nature 529~(7587) (2016) 484--489.

\bibitem{Goertzel2014}
B.~Goertzel, Artificial general intelligence: concept, state of the art, and
  future prospects, Journal of Artificial General Intelligence 5~(1) (2014).

\bibitem{Chollet2020}
F.~Chollet, A definition of intelligence for the real world, Journal of
  Artificial General Intelligence 11~(2) (2020) 27--30.

\bibitem{Johnson2021}
A.~Johnson, W.~K. Vong, B.~Lake, T.~Gureckis, Fast and flexible: {Human}
  program induction in abstract reasoning tasks, arXiv preprint
  arXiv:2103.05823 (2021).

\bibitem{FlashFill2013}
A.~Menon, O.~Tamuz, S.~Gulwani, B.~Lampson, A.~Kalai, A machine learning
  framework for programming by example, in: Int. Conf. Machine Learning, PMLR,
  2013, pp. 187--195.

\bibitem{Rissanen1978}
J.~Rissanen, Modeling by shortest data description, Automatica 14~(5) (1978)
  465--471.

\bibitem{Grunwald2019}
P.~Gr{\"u}nwald, T.~Roos, Minimum description length revisited, International
  journal of mathematics for industry 11~(01) (2019).

\bibitem{Fer2023dexa}
S.~Ferré, Dexteris: Data exploration and transformation with a guided query
  builder approach, in: Int. Conf. Database and Expert Systems Applications,
  Springer, 2023, pp. 361--376.

\bibitem{Fer2024ida}
S.~Ferré, Tackling the abstraction and reasoning corpus ({ARC}) with
  object-centric models and the {MDL} principle, in: Int. Symp. Intelligent
  Data Analysis, Springer, 2024, pp. 3--15.

\bibitem{Fischer2020}
R.~Fischer, M.~Jakobs, S.~M\"{u}cke, K.~Morik, Solving {Abstract} {Reasoning}
  {Tasks} with {Grammatical} {Evolution}., in: {LWDA}, CEUR-WS 2738, 2020, pp.
  6--10.

\bibitem{Alford2021cnta}
S.~Alford, A.~Gandhi, A.~Rangamani, A.~Banburski, T.~Wang, S.~Dandekar,
  J.~Chin, T.~Poggio, P.~Chin, Neural-guided, bidirectional program search for
  abstraction and reasoning, in: Int. Conf. Complex Networks and Their
  Applications, Springer, 2021, pp. 657--668.

\bibitem{Ouellette2024}
S.~Ouellette, Towards efficient neurally-guided program induction for arc-agi,
  arXiv preprint arXiv:2411.17708 (2024).

\bibitem{Xu2022}
Y.~Xu, E.~B. Khalil, S.~Sanner, Graphs, constraints, and search for the
  abstraction and reasoning corpus, arXiv preprint arXiv:2210.09880 (2022).

\bibitem{Ainooson2023}
J.~Ainooson, D.~Sanyal, J.~P. Michelson, Y.~Yang, M.~Kunda, An approach for
  solving tasks on the abstract reasoning corpus, arXiv preprint
  arXiv:2302.09425 (2023).

\bibitem{Greenblatt2024}
R.~Greenblatt,
  \href{https://redwoodresearch.substack.com/p/getting-50-sota-on-arc-agi-with-gpt}{Getting
  50\%(sota) on arc-agi with gpt-4o} (2024).
\newline\urlprefix\url{https://redwoodresearch.substack.com/p/getting-50-sota-on-arc-agi-with-gpt}

\bibitem{Berman2024}
J.~Berman,
  \href{https://jeremyberman.substack.com/p/how-i-got-a-record-536-on-arc-agi}{How
  i came in first on arc-agi-pub using sonnet 3.5 with evolutionary test-time
  compute} (2024).
\newline\urlprefix\url{https://jeremyberman.substack.com/p/how-i-got-a-record-536-on-arc-agi}

\bibitem{Chollet2019}
F.~Chollet, On the measure of intelligence, arXiv preprint arXiv:1911.01547
  (2019).

\bibitem{Acquaviva2022}
S.~Acquaviva, Y.~Pu, M.~Kryven, T.~Sechopoulos, C.~Wong, G.~Ecanow, M.~Nye,
  M.~Tessler, J.~Tenenbaum, Communicating natural programs to humans and
  machines, Advances in Neural Information Processing Systems 35 (2022)
  3731--3743.

\bibitem{Lieberman2001pbe}
H.~Lieberman, Your Wish is My Command, The Morgan Kaufmann series in
  interactive technologies, Morgan Kaufmann / Elsevier, 2001.

\bibitem{MugRae1994}
S.~Muggleton, L.~D. Raedt, Inductive logic programming: Theory and methods,
  Journal of Logic Programming 19,20 (1994) 629--679.

\bibitem{Polozov2015flashmeta}
O.~Polozov, S.~Gulwani, Flashmeta: A framework for inductive program synthesis,
  in: ACM SIGPLAN Int. Conf. Object-Oriented Programming, Systems, Languages,
  and Applications, 2015, pp. 107--126.

\bibitem{Mulleners2023}
N.~Mulleners, J.~Jeuring, B.~Heeren, Program synthesis using example
  propagation, in: Int. Symp. Practical Aspects of Declarative Languages,
  Springer, 2023, pp. 20--36.

\bibitem{Rule2024}
J.~S. Rule, S.~T. Piantadosi, A.~Cropper, K.~Ellis, M.~Nye, J.~B. Tenenbaum,
  Symbolic metaprogram search improves learning efficiency and explains rule
  learning in humans, Nature Communications 15~(1) (2024) 6847.

\bibitem{Ellis2021dreamcoder}
K.~Ellis, et~al., Dreamcoder: Bootstrapping inductive program synthesis with
  wake-sleep library learning, in: ACM Int. Conf. Programming Language Design
  and Implementation, 2021, pp. 835--850.

\bibitem{LakSalTen2015}
B.~M. Lake, R.~Salakhutdinov, J.~B. Tenenbaum, Human-level concept learning
  through probabilistic program induction, Science 350~(6266) (2015)
  1332--1338.

\bibitem{MCTS2012}
C.~B. Browne, E.~Powley, D.~Whitehouse, S.~M. Lucas, P.~I. Cowling,
  P.~Rohlfshagen, S.~Tavener, D.~Perez, S.~Samothrakis, S.~Colton, A survey of
  monte carlo tree search methods, IEEE Transactions on Computational
  Intelligence and AI in games 4~(1) (2012) 1--43.

\bibitem{Burke2017}
E.~K. Burke, Y.~Bykov, The late acceptance hill-climbing heuristic, Eu. J.
  Operational Research 258~(1) (2017) 70--78.

\end{thebibliography}




\FloatBarrier

\end{document}